\documentclass{article}
\pdfoutput=1
\usepackage{etoolbox}

\newtoggle{arxiv}
\newtoggle{icml}
\newtoggle{neuripscausal}
\toggletrue{arxiv}
\togglefalse{icml}
\togglefalse{neuripscausal}

\iftoggle{arxiv}{
  \usepackage{url}            %
  \usepackage{fullpage}
  \usepackage{booktabs}       %
  \usepackage{amsfonts}       %
  \usepackage{nicefrac}       %
  \usepackage{graphicx} %
  \usepackage{tikz}
  \usepackage{amsmath}
  \usepackage{amsthm}
  \usepackage{amssymb}
  \usepackage{comment}
  \usepackage{color-edits}
  \usepackage[noend]{algorithmic}
  \usepackage[ruled,vlined]{algorithm2e}
  \addauthor{vs}{red}
  \addauthor{yd}{blue}
  \addauthor{lm}{green}
  \usepackage{natbib}

 \author{Yash Deshpande\thanks{Department of Mathematics, Massachusetts Institute
 of Technology and Microsoft Research} \;\and \; Lester Mackey\thanks{Microsoft Research} \; \and\;
 Vasilis Syrgkanis\thanks{Microsoft Research} \;\and \; Matt Taddy\thanks{Chicago Booth and Microsoft Research}} 
}{}

\iftoggle{icml}{
\usepackage{graphicx}
\usepackage{subfigure}
\usepackage{booktabs} %
\usepackage[breaklinks]{hyperref}
\usepackage[hyphenbreaks]{breakurl}
\usepackage[accepted]{icml2018}
\usepackage{amsmath,amsthm,amssymb}
\icmltitlerunning{Accurate Inference for Adaptive Linear Models}
}{}

\iftoggle{neuripscausal}{
 \usepackage{url}            %
  \usepackage{booktabs}       %
  \usepackage{amsfonts}       %
  \usepackage{nicefrac}       %
  \usepackage{graphicx} %
  \usepackage{tikz}
  \usepackage{amsmath}
  \usepackage{amsthm}
  \usepackage{amssymb}
  \usepackage{comment}
  \usepackage{color-edits}
  \usepackage[noend]{algorithmic}
  \usepackage[ruled,vlined]{algorithm2e}
  \addauthor{vs}{red}
  \addauthor{yd}{blue}
  \addauthor{lm}{green}
  \usepackage{natbib} 
  \usepackage{neurips_2019}
  \usepackage{microtype}

\author{Yash Deshpande\thanks{Institute for Data, Systems and Society, Massachusetts Institute
 of Technology} \;\and \; Lester Mackey\thanks{Microsoft Research} \; \and\;
 Vasilis Syrgkanis\thanks{Microsoft Research} \;\and \; Matt Taddy\thanks{Chicago Booth and Microsoft Research}}

}{}

\newtheorem{theorem}{Theorem}
\newtheorem{lemma}[theorem]{Lemma} 
\newtheorem{proposition}[theorem]{Proposition} 

\newtheorem{corollary}[theorem]{Corollary}
\newtheorem{definition}{Definition}

\newtheorem{assumption}{Assumption}

\def\<{\langle}
\def\>{\rangle}
\def\mubar{{\bar\mu}}
\def\lambdamin{{\lambda_{\min}}}

\def\im{{\mathrm{i}}}

\def\eps{{\varepsilon}}

\def\id{{\rm I}}
\def\sT{{\sf T}}

\def\sb{{\sf b}}
\def\sv{{\sf v}}

\def\P{{\mathbb P}}

\def\E{{\mathbb E}} %

\def\reals{\mathbb{R}}

\def\normal{{\sf N}}

\def\Unif{{\sf Unif}}

\def\Tr{{\sf {Tr}}}

\def\cG{\mathcal{G}}

\def\Var{{\sf Var}}

\def\ind{\mathbb{I}}

\newcommand\norm[1]{\left\lVert{#1}\right\rVert}
\newcommand\abs[1]{\left\lvert{#1}\right\rvert}

\def\cX{{\cal X}}
\def\cF{{\cal F}}

\def\mle{{\, \preccurlyeq \,}}
\def\mge{{\, \succcurlyeq \,}}

\def\Var{{\rm Var}}

\def\convD{{\,\stackrel{\mathrm{d}}{\Rightarrow} \,}}

\def\ols{{\sf OLS}}

\def\betahat{{\widehat{\beta}}}

\def\bx{{\boldsymbol x}}

\def\bv{{\boldsymbol v}}

\def\bX{{\boldsymbol X}}

\DeclareMathAlphabet{\mathpzc}{OT1}{pzc}{m}{it}

\newcommand{\vect}[1]{\boldsymbol{#1}}

\graphicspath{{figs/}}

\newenvironment{talign*}
 {\csname align*\endcsname}
 {\endalign}
\newenvironment{talign}
 {\csname align\endcsname}
 {\endalign}

\newcommand{\twonorm}[1]{\norm{#1}_2} %
\def\staticnorm#1{\|{#1}\|} %
\newcommand{\statictwonorm}[1]{\staticnorm{#1}_2} %
\newcommand{\opnorm}[1]{\norm{#1}_{op}} %

\usepackage{cleveref}
\crefname{equation}{}{}
\usepackage{autonum}
\usepackage{color}

\title{Accurate Inference in Adaptive Linear Models}

\begin{document}
\ifboolexpr{togl{neuripscausal}}{
  \maketitle
}{}

\ifboolexpr{togl{arxiv}}{
\maketitle
}{}

\iftoggle{icml}{

\twocolumn[
\icmltitle{Accurate Inference for Adaptive Linear Models}

\begin{icmlauthorlist}
\icmlauthor{Yash Deshpande}{mit}
\icmlauthor{Lester Mackey}{msr}
\icmlauthor{Vasilis Syrgkanis}{msr}
\icmlauthor{Matt Taddy}{msr,booth}
\end{icmlauthorlist}

\icmlaffiliation{mit}{Department of Mathematics, MIT}
\icmlaffiliation{msr}{Microsoft Research New England}
\icmlaffiliation{booth}{Booth School of Business, University of Chicago}

\icmlcorrespondingauthor{Yash Deshpande}{yash@mit.edu}

\icmlkeywords{Machine Learning, ICML}

\vskip 0.3in
]
\printAffiliationsAndNotice{} 
}{}

\begin{abstract}
Estimators computed from adaptively collected data do not behave like their non-adaptive brethren.
Rather, the sequential dependence of the collection policy can lead to severe distributional biases that persist even in the infinite data limit.
We develop a general method 
-- \emph{$\vect{W}$-decorrelation} -- for transforming the bias of adaptive linear regression estimators into variance.
The method uses only coarse-grained information about the data collection policy and does not need access to propensity scores or exact knowledge of the policy.
We bound the finite-sample bias and variance of the $\vect{W}$-estimator and develop asymptotically correct confidence intervals based on a novel martingale central limit theorem. 
We then demonstrate the empirical benefits of the generic $\vect{W}$-decorrelation procedure in two different adaptive data settings: the multi-armed bandit and the autoregressive time series.
\end{abstract}

\section{Introduction}
Consider a dataset of $n$ sample points $(y_i, \vect{x}_i)_{i\le n}$
where $y_i$ represents an observed outcome and $\vect{x}_i \in \reals^p$ an associated vector of covariates. 
In the standard linear model, the outcomes 
and covariates are related through a parameter $\beta$:
\begin{align}\label{eq:exampleLinearModel}
y_i &= \< \vect{x}_i, \beta\> + \eps_i.
\end{align}
In this model, the `noise' term $\eps_i$ represents inherent variation in the sample, 
or the variation that is not captured in the model. Parametric
models 
of the type \cref{eq:exampleLinearModel} are a fundamental building block in many
machine learning problems. A common additional
assumption is that the covariate vector $\vect{x}_i$ for a given datapoint $i$ is independent of 
the other sample point outcomes $(y_j)_{j\ne i}$ and the inherent variation $(\eps_j)_{j\in [n]}$. 
This paper is motivated by experiments where the
sample $(y_i, \vect{x}_i)_{i\le n}$ is not completely randomized but rather \emph{adaptively} chosen. By adaptive, 
we mean that the choice of the data point $(y_i, \vect{x}_i)$ is
guided from inferences on past data $(y_j, \vect{x}_j)_{j < i}$. Consider the following sequential paradigms:

\begin{enumerate}
\item Multi-armed bandits: This class of sequential decision making problems captures the classical `exploration versus exploitation' tradeoff. At each time $i$, the experimenter
chooses an `action' $\vect{x}_i$ from a set of available actions $\cX$ and accrues a reward $R(y_i)$ where $(y_i, \vect{x}_i)$ follow the model \cref{eq:exampleLinearModel}. Here the
experimenter must balance the conflicting goals of learning about the underlying model (i.e., $\beta$) for better future rewards, while still accruing reward in the current 
time step. 
\item Active learning: Acquiring labels  $y_i$ is potentially
costly, and the experimenter aims to learn with as few outcomes as possible. 
At time $i$, based on prior data $(y_j, \vect{x}_j)_{j\le i-1}$ the experimenter 
chooses a new data point $\vect{x}_i$ to label based on its value in learning.
\item Time series analysis: Here, the data points $(y_i, \vect{x}_i)$ are naturally ordered in time, with $(y_i)_{i\le n}$ denoting a time series
and the covariates $\vect{x}_i$ include observations from the prior time points. 
\end{enumerate}

Here, time induces a natural sequential dependence across the samples.
In the first two instances, the actions or policy of the experimenter 
are responsible for creating such dependence. In the case of time 
series data, this dependence is endogenous and a consequence of the modeling. 
A common feature, however, is that the choice of the {design} 
or sequence $(\vect{x}_i)_{i\le n}$ is typically not
made for inference on the model after the data collection is completed. 
This does not, of course, imply that accurate estimates on the parameters
$\beta$ cannot be made from the data. Indeed, 
it is often the case that the sample is informative enough to extract 
consistent estimators of the underlying parameters. Indeed, this is often 
crucial to the success of the experimenter's policy.
For instance, `regret' in sequential decision-making
or risk in active learning are intimately connected with the accurate estimation
of the underlying parameters \citep{castro2008minimax,audibert2009minimax,bubeck2012regret,rusmevichientong2010linearly} . Our motivation is the natural
follow-up question of accurate \emph{ex post} inference in the standard statistical sense:
\begin{center}
\begin{quote}
Can adaptive data be used to compute accurate confidence regions and $p$-values?
\end{quote}
\end{center}

As we will see, the key challenge is
that even in the simple linear model of \cref{eq:exampleLinearModel}, the distribution of classical estimators can differ from the predicted central limit behavior of non-adaptive designs. In
this context we make the following contributions:

\begin{itemize}
\item \textbf{Decorrelated estimators:} We present a 
general method to decorrelate arbitrary estimators $\betahat(\vect{y}, \vect{X}_n)$ constructed from the data. This construction admits a simple decomposition
into a `bias' and `variance' term. In comparison
with competing methods, like propensity weighting, our proposal
requires little explicit information about the data-collection policy.  

\item \textbf{Bias and variance control:} Under a natural exploration 
condition on the data collection policy, we establish that the bias
and variance can be controlled at nearly optimal levels. In the
multi-armed bandit setting, we prove this under an especially weak averaged
exploration condition.  

\item \textbf{Asymptotic normality and inference:} We establish a martingale
central limit theorem (CLT) under a moment stability assumption. 
Applied to our decorrelated estimators, this allows 
us to construct confidence intervals and conduct hypothesis tests in the 
usual fashion. 

\item \textbf{Validation:} We demonstrate the usefulness of
the decorrelating construction in two different scenarios: 
multi-armed bandits (MAB) and autoregressive (AR) time series. We observe that our decorrelated estimators
retain expected central limit behavior in 
regimes where the standard estimators do not, thereby facilitating
accurate inference. 

\end{itemize}

The rest of the paper is organized with our main results in Section \ref{sec:model}, discussion of related work in Section \ref{sec:related}, and experiments in Section \ref{sec:expt}. 
\iftoggle{neuripscausal}{
An earlier version of this paper was published in ICML 2018 (citation
retracted). This version contains a new `limited information'
martingale central limit theorem, as well as new results on
for the special case of multi-armed bandits.}{}
\iftoggle{icml}{
Due to page constraints, all proofs are given in Appendix \ref{app:proofs} in the supplementary
information.
}{}

\section{Main results: $\vect{W}$-decorrelation} \label{sec:model}

We focus on the linear
model and assume that the
data pairs $(y_i, \vect{x}_i)$ satisfy:
\begin{talign} \label{eq:linearmodel}
y_i &= \<\vect{x}_i, \beta\> + \eps_i, 
\end{talign}
where $\eps_i$ are independent and identically distributed
random variables with $\E\{\eps_i \} = 0$, $\E\{\eps_i^2\} = \sigma^2$ 
and bounded third moment. We assume that the samples
are ordered naturally in time and let $\{\cF_i\}_{i\ge 0}$ denote 
the filtration representing
the sample. Formally,
we let data points $(y_i, \vect{x}_i)$ be
adapted to this filtration, i.e. $(y_i, \vect{x}_i)$ are
measurable with respect to $\cF_j$ for all $j\ge i$. 

Our goal in this paper is to use the available data to 
construct \emph{ex post} confidence intervals and $p$-values
for individual parameters, i.e. entries of $\beta$. 
A natural starting point is to consider is the standard least squares
estimate: 
\begin{talign}
\betahat_\ols = (\vect{X}_n^\sT \vect{X}_n)^{-1}\vect{X}_n^\sT \vect{y}_n,  
\end{talign}
where $\vect{X}_n = [\vect{x}_1^\sT, \dots \vect{x}_n^\sT] \in\reals^{n\times p} $ is the 
design matrix and $\vect{y}_n = [y_1, \dots y_n]\in\reals^n$.
When data collection is 
non-adaptive, classical results imply that
the standard least squares estimate $\betahat_{\ols}$ is distributed
asymptotically as $\normal(\beta, \sigma^2 (\vect{X}_n^\sT\vect{X}_n)^{-1})$, 
where $\normal(\mu, \Sigma)$ denotes the Gaussian distribution with mean
$\mu$ and covariance $\Sigma$. \citet{lai1982least} extend these results
to the current scenario:

\begin{theorem}[Theorems 1, 3 \citep{lai1982least}] \label{thm:lai}
Let $\lambda_{\min}(n)$ ($\lambda_{\max}(n)$) denote
the minimum (resp. maximum) eigenvalue of $\vect{X}_n^\sT\vect{X}_n$. 
Under the model \cref{eq:linearmodel}, assume that $(i)$ $\eps_i$ have finite third
moment and  $(ii)$ almost surely, $\lambda_{\min}(n)\to \infty$ with $\lambda_{\min } = \Omega(\log\lambda_{\max})$ and $(iii)$ $\log\lambda_{\max} = o(n)$. Then the following limits hold almost surely:
\begin{talign}
\lVert{\betahat_{\ols} - \beta}\rVert^2_2 &\le C\frac{\sigma^2 p\log \lambda_{\max}}{\lambda_{\min}} \\
  \lvert \frac{1}{n \sigma^2}\lVert{ \vect{y}_n - \vect{X}_n \betahat_{\ols}}\rVert^2_2 - 1  \rvert &\le C(p)\frac{1+\log\lambda_{\max}}{n}. 
\end{talign}
Further assume the following stability condition: 
there exists a deterministic sequence of matrices
$\vect{A}_n$ such that $(iii)$ $\vect{A}_n^{-1}(\vect{\vect{X}_{n}}^\sT\vect{X}_n)^{1/2} \to \id_p $
and $(iv)$ $\max_{i} \twonorm{\vect{A}_n^{-1}\vect{x}_i} \to 0$ in probability. Then,
\begin{talign}
(\vect{X}_n^\sT\vect{X}_n)^{1/2} (\betahat_{\ols} - \beta) &\convD \normal(0, \sigma^2\id_p). 
\end{talign}
\end{theorem}
At first blush, this allows to construct confidence regions
in the usual way. More precisely, the result implies that $\widehat{\sigma}^2 = \lVert \vect{y_n} - \vect{X}_n \betahat_{\ols} \rVert_2^2/n $ is a consistent estimate
of the noise variance. Therefore, the interval $[\betahat_{\ols, 1} - 1.96 \widehat{\sigma} (\vect{X}_n^\sT\vect{X}_n)^{-1}_{11}, \betahat_{\ols, 1} + 1.96 \widehat{\sigma} (\vect{X}_n^\sT\vect{X}_n)^{-1}_{11}]$ is a 95\% two-sided confidence interval for the first coordinate $\beta_1$. Indeed, this result is sufficient for
a variety of scenarios with weak dependence across samples, 
such as when the $(y_i, \vect{x}_i)$ form a Markov chain 
that mixes rapidly.  However, while the assumptions for
consistency are minimal, the additional stability assumption
required for asymptotic normality poses some challenges.
In particular:

\begin{enumerate}
\item The stability condition can 
provably fail to hold for scenarios where the dependence
across samples is non-negligible. This is not a weakness of Theorem \ref{thm:lai}:  the CLT need not hold for the $\ols$ estimator \citep{lai1982least, lai1983fixed}. 
\item The rate of convergence to the asymptotic CLT depends on the \emph{quantitative
rate} of the stability condition. In other words,
variability in the inverse covariance 
$\vect{X}_n^\sT\vect{X}_n$ can cause 
deviations from normality of $\ols$ estimator \citep{dvoretzky1972asymptotic}. 
In finite samples, this can manifest itself in the bias of the $\ols$ 
estimator as well as in higher moments. 
\end{enumerate}

An example of this phenomenon 
is the standard multi-armed bandit problem \citep{lai1985asymptotically}. At each time
point $i\le n$, the experimenter (or data collecting policy) chooses an arm
$k\in\{1, 2, \dots , p\}$ and observes a reward $y_i$ with mean $\beta_k$. With
$\beta\in\reals^p$ denoting the mean rewards, this falls within the scope of
model \cref{eq:linearmodel}, where the vector $\vect{x}_i$ takes the value $\vect{e}_k$ (the $k^{\text{th}}$
basis vector), if the $k^{\text{th}}$ arm or option is chosen at time $i$.\footnote{Strictly speaking, the model
\cref{eq:linearmodel} assumes that the errors have the same variance, which need not be true
for the multi-armed bandit as discussed. We focus on the homoscedastic case where the errors have the same variance in this paper.} Other stochastic bandit problems with covariates such as contextual or linear bandits \citep{rusmevichientong2010linearly, li2010contextual, deshpande2012linear} can also be incorporated fairly naturally
into our framework. For the purposes of this paper, however, we restrict ourselves to the simple case 
of multi-armed bandits without covariates. In this setting, ordinary least squares estimates 
correspond to computing sample means for each arm. The stability condition of 
Theorem \ref{thm:lai} requires that $N_k(n)$, the number of times a specific
arm $k\in [p]$ is sampled is asymptotically deterministic as $n$ grows large.
This is true for certain regret-optimal algorithms \citep{russo2016simple,garivier2011kl}. 
Indeed, for such algorithms, as the sample size $n$ grows large, the suboptimal
arm is sampled $N_k(n) \sim C_k(\beta) \log n$ 
for a constant $C_k(\beta)$ that depends on $\beta$ and the distribution of noise $\eps_i$. 
However, in finite samples, the dependence on $C_k(\beta)$ and the slow
convergence rate of $(\log n)^{-1/2}$ lead to significant deviation
from the expected central limit behavior. 

\citet{villar2015multi} studied a variety of multi-armed bandit algorithms in the
context of clinical trials. They empirically demonstrate that sample mean estimates 
from data collected using many standard
multi-armed bandit algorithms are biased. Recently, 
\cite{nie2017why} proved that this bias is negative for Thompson sampling and UCB. 
The presence of bias in sample means demonstrates that standard methods for inference, as advocated
by Theorem \ref{thm:lai}, can be misleading when the same data is now used for inference. 
As a pertinent example,  testing the hypotheses ``the mean reward of arm 1 exceeds that of 2''
based on classical theory can be significantly affected by adaptive data collection. 

The papers \citep{villar2015multi,nie2017why} focus on the finite sample effect of the data collection policy on the bias and suggest methods to reduce the bias. It is not hard
to find examples where higher moments or tails of the distribution can be influenced by the data collecting
policy. A simple, yet striking, example 
is the standard autoregressive model (AR) for time series data. In its 
simplest form, the AR model has one covariate, i.e. $p=1$ with 
$\vect{x}_i = y_{i-1}$. In this case:
\begin{talign}\label{eq:ar1model}
y_{i} &= \beta y_{i-1}+ \eps_i. 
\end{talign}
Here the least squares estimate is given by $\betahat_{\ols} = \sum_{i \le n-1} y_{i+1} y_{i}/\sum_{i \le n-1} y_{i-1}^2$. When $|\beta|$ is bounded away from
1, the series is asymptotically stationary and the $\ols$ estimate has Gaussian
tails. On the other hand, when $\beta - 1$ is on the order of $1/n$ the limiting
distribution of the least squares estimate is non-Gaussian and dependent on the gap
$\beta - 1$ (cf. \cite{chan1987asymptotic}). A histogram
for the normalized 
$\ols$ errors in two cases: $(i)$ stationary with $\beta = 0.02$
and $(ii)$  nonstationary with $\beta = 1.0$ is shown on the left in
Figure \ref{fig:ar1demonstration}. The $\ols$ estimate yields
clearly non-Gaussian errors when nonstationary, i.e. when $\beta$
is close to 1.  

On the other hand, \emph{using the same data} our decorrelating procedure is able
to obtain estimates admitting Gaussian limit distributions, as evidenced
in the right panel of Figure \ref{fig:ar1demonstration}. We show a similar phenomenon
in the MAB setting where our decorrelating procedure corrects for
the unstable behavior of the $\ols$ estimator (see Section \ref{sec:expt} for details
on the empirics).
Delegating discussion of further related
work to \ref{sec:related}, we now describe this procedure and its motivation.

\begin{figure*}[t]
\centering
\includegraphics[width=0.45\linewidth]{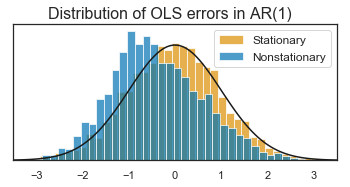}
\hfill
\includegraphics[width=0.45\linewidth]{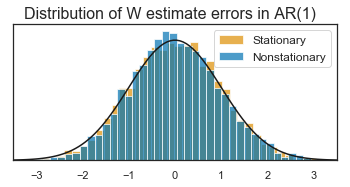}
	\caption{The distribution of normalized errors for (left) the $\ols$ 
	estimator for stationary and (nearly) nonstationary AR(1) time
    series and (right) error distribution
	for both models after decorrelation. \label{fig:ar1demonstration}}
\end{figure*}

\subsection{Removing the effects of adaptivity}

We propose to decorrelate the $\ols$ estimator by constructing:
\begin{talign}\label{eq:decorrdef}
\betahat^d &= \betahat_\ols + \vect{W}_n (y - \vect{X}_n \betahat_\ols),
\end{talign}
for a specific choice of a `decorrelating' or `whitening' matrix
$\vect{W}_n\in\reals^{p\times n}$. This is inspired by the high-dimensional linear regression debiasing constructions of
\cite{zhang2014confidence,javanmard2014hypothesis, javanmard2014confidence, van2014asymptotically}. As we will see, this construction is useful also in the present regime
where we keep $p$ fixed and $n\gtrsim p$. By rearranging:
\begin{talign}
\betahat^d - \beta &= (\id_p - \vect{W}_n \vect{X}_n)(\betahat_\ols - \beta)
 + \vect{W}_n \vect{\eps}_n \nonumber\nonumber\\
&\equiv  \sb + \sv. 
\end{talign}
We interpret $\sb$ as a `bias' and $\sv$ as a `variance'. This is
based on the following critical constraint on the construction
of the whitening matrix 
$\vect{W}_n$: 

\begin{definition}[Well-adaptedness of $\vect{W}_n$]\label{assmp:predictability} Without loss of generality, we assume that $\eps_i$ are adapted to
$\cF_i$. Let $\cG_i\subset\cF_i$ be a filtration such that $\vect{x}_i$ are adapted w.r.t. 
$\cG_i$ and $\eps_i$ is independent of $\cG_i$. 
We say that $\vect{W_n}$ is well-adapted if the columns of $\vect{W}_n$ are adapted to $\cG_i$, i.e. the $i^\text{th}$ column $\vect{w}_i$ is measurable with respect to $\cG_{i}$. 
\end{definition}

With this in hand, we have the following simple lemma.
\begin{lemma}\label{lem:biasvardecomp}
    Assume $\vect{W}_n$ is well-adapted. Then: 
\begin{talign}
    \statictwonorm{\beta - \E\{\betahat^d\}} &\le \E\{\twonorm{ \sb}\},  
    \\ \Var(\sv) &= \sigma^2 \E\{\vect{W}_n\vect{W}_n^\sT \vect\}. 
\end{talign}
 \end{lemma} 
A concrete proposal is to trade-off the bias, controlled by the size of
$\id_p - \vect{W}_n\vect{X}_n$, with the the variance which appears through $\vect{W}_n\vect{W}_n^\sT$. This leads to the following optimization problem:
\begin{talign}\label{eqn:W_bias_variance_tradeoff}
    \vect{W}_n & = \arg\min _{\vect{W}} \lVert  \id_p - \vect{W}\vect{X}_n  \rVert_F^2 + \lambda \Tr (\vect{W} \vect{W}^\sT ). 
\end{talign}
Solving the above in closed form yields ridge estimators for $\beta$, 
and by continuity, also the standard least squares estimator. Departing 
from \cite{zhang2014confidence,javanmard2014confidence}, we solve the above
in an \emph{online} fashion in order to obtain a well-adapted $\vect{W}_n$. 
We define, $\vect{W}_0 = 0$, $\vect{X}_0 = 0$, and recursively $\vect{W}_n = [\vect{W}_{n-1} \vect{w}_n]$  for
\begin{align}
\vect{w}_n & = \arg\min_{\vect{w} \in \reals^p} \|\id - \vect{W}_{n-1}\vect{X}_{n-1} - \vect{w} \vect{x}_n^T\|_F^2 + \lambda \twonorm{\vect{w}}^2.
\end{align}
As in the case of the offline optimization, we may obtain closed
form formulae for the columns $\vect{w}_i$ (see Algorithm
\ref{alg:decorrmethod}). The method as specified requires
$O(np^2)$ additional computational overhead, which is typically minimal
compared to computing $\betahat_\ols$ or a regularized
version like the ridge or lasso estimate.  We refer to 
$\betahat^d$ as a \emph{$\vect{W}$-estimate} or a
\emph{$\vect{W}$-decorrelated estimate}.

\iftoggle{arxiv}{

\subsection{Interpretation as reverse implicit SGD}
While we motivated $\vect{W}$-decorrelation
decorrelation
as an online procedure for optimizing the bias-variance tradeoff objective \cref{eqn:W_bias_variance_tradeoff}, it holds a dual interpretation as {implicit stochastic gradient descent} (SGD) \citep[see, e.g.,][]{kulis2010implicit}, also known as {incremental proximal minimization} \citep{bertsekas2011incremental} or the {normalized least mean squares filter} \citep{nagumo1967learning} in this context, with step-size $\lambda$ applied to the least-squares objective, $\frac{1}{n} \sum_{i=1}^n (y_i - \<{\beta},{\vect{x}_i}\>)^2$.
Importantly, to obtain the well-adapted form of our updates, one must apply implicit SGD \emph{in reverse}, starting with the final observation $(y_n, \vect{x}_n)$ and ending with the initial observation $(y_1, \vect{x}_1)$; this recipe yields the parameter updates $\betahat_0 = \betahat_\ols$ and
\begin{talign}
\betahat_{i+1} &= \betahat_{i} + \vect{x}_{n-i} (y_{n-i} - \<\vect{x}_{n-i}, \betahat_{i+1}\>)/\lambda \\
    &= (\id_p + \vect{x}_{n-i}\vect{x}_{n-i}^\sT/\lambda)^{-1} (\betahat_{i} + y_{n-i}\vect{x}_{n-i} /\lambda) \\
    &= (\id_p - \vect{x}_{n-i}\vect{x}_{n-i}^\sT/(\lambda + \twonorm{\vect{x}_{n-i}}^2)) \betahat_{i} \\
     &\quad+ y_{n-i}\vect{x}_{n-i} /(\lambda +\twonorm{\vect{x}_{n-i}}^2).%
\end{talign}
Unrolling the recursion, we obtain 
$\betahat_n = \betahat_\ols + \sum_{i=1}^n y_i \vect{w}_i$
with each $\vect{w}_i$ precisely as in Algorithm \ref{alg:decorrmethod}:
\begin{talign}
\textstyle
\vect{w}_i 
       = \Big(\prod_{j=1}^{i-1}(\id_p - \vect{x}_j\vect{x}_j^\sT/(\lambda + \twonorm{\vect{x}_j}^2) \Big) \cdot \vect{x}_i.  
\end{talign}
}{}

\subsection{Bias and variance}
We now examine the bias and variance control 
for $\betahat^d$. We first begin with a general bound for 
the variance:

\begin{theorem}[Variance control] \label{thm:variance}
For any $\lambda\ge1$ set non-adaptively, we have that
    \begin{talign}
        \Tr\{\Var(\sv)\}  &\le \frac{\sigma^2}{\lambda}(p - \E\{\lVert \id_p - \vect{W}_n\vect{X}_n\rVert_F^2\}). 
    \end{talign}
In particular, $\Tr\{\Var(\sv)\}\le \sigma^2 p/\lambda$. Further, if $\twonorm{\vect{x}_i}^2 \le C$ for all $i$:
\begin{talign}
\Tr\{\Var(\sv)\} &\asymp \frac{\sigma^2}{\lambda}(p - \E\{\lVert \id_p - \vect{W}_n\vect{X}_n\rVert_F^2\}). 
\end{talign}
\end{theorem}
This theorem suggests that one must set $\lambda$ as large
as possible to minimize the variance. While this is accurate,
one must take into account the bias of $\betahat^d$ and its
dependence on the regularization $\lambda$. Indeed, for large
$\lambda$, one would expect that $\id_p - \vect{W}_n\vect{X}_n \approx \id_p$, 
which would not help control the bias. In general,  one
would hope to set $\lambda$, thereby determining $\betahat^d$,
at a level where its bias is negligible in comparison to
the variance. The following theorem formalizes this: 

    \begin{theorem}[Variance dominates MSE] \label{thm:vardom}
    Recall that the matrix $\vect{W}_n$ is a function of $\lambda$. 
    Suppose that there exists a deterministic sequence $\lambda(n)$ such
    that: 
    \begin{talign}
        \E \{\opnorm{ \id_p - \vect{W}_n\vect{X}_n}^2\} = o(1/\log n), \label{eq:vardomcond1}\\
        \P\{\lambda_{\min}(\vect{X}_n^\sT\vect{X}_n) \le \lambda(n) \log\log n \} \le 1/n, \label{eq:vardomcond2}
    \end{talign}
    Then we have
    \begin{align}
        \frac{\twonorm{\E\{{\sb}\}}^2}{ \Tr\{\Var(\sv)\} }&= o(1) .
    \end{align}

\end{theorem}
The conditions of Theorem \ref{thm:vardom}, in particular the bias
condition on $\id_p - \vect{W}_n\vect{X}_n$ are quite general. In the following proposition, we verify
some sufficient conditions under which the premise of Theorem \ref{thm:vardom}
hold.
\begin{proposition}\label{prop:sufficient}
Either of the following conditions suffices for the requirements of Theorem \ref{thm:vardom}. 
\begin{enumerate}
\item The data collection policy satisfies for some sequence $\mu_n(i)$  and for all $\lambda \ge 1$:
\begin{talign}
     \E\{ \frac{\vect{x}_i\vect{x}_i^\sT} {\lambda + \twonorm{\vect{x}_n}^2} \vert \cG_{i-1} \}
     &\mge \frac{\mu_n(i)}{\lambda} \id_p, \label{eq:explorecond1} \\
     \sum_i \mu_n(i) \equiv n\mubar_n &\ge K\sqrt{n}, \label{eq:explorecond2}
\end{talign}
for a large enough constant $K$. 
Here we keep $\lambda(n) \asymp n\mubar_n /\log(p\log n)$. 
\item The matrices $(\vect{x}_i\vect{x}_i^\sT)_{i\le n}$ commute and $\lambda(n)\log\log n$ is (at most) the
$1/n^\text{th}$ percentile of $\lambda_{\min}(\vect{X}_n^\sT\vect{X}_n)$. 
\end{enumerate}
\end{proposition}

It is useful to consider the intuition for the sufficient conditions
given in Proposition \ref{prop:sufficient}. By Lemma \ref{lem:biasvardecomp}, note that
the bias is controlled by 
$\opnorm{\id - \vect{W}_n\vect{X}_n}$, which increases 
with $\lambda$. Consider a case in which 
the samples
$\vect{x}_i$ lie in a strict subspace of $\reals^p$. In this case, controlling the bias
uniformly over $\beta\in\reals^p$ is now impossible regardless of the choice of $\vect{W}_n$. 
 For example, in a multi-armed bandit problem, if the policy does not 
sample a specific arm, there is no information available about the reward
distribution of that arm. Proposition \ref{prop:sufficient}
the intuition that the data collecting policy should explore
the full parameter space. 
For multi-armed bandits, policies such
as epsilon-greedy and Thompson sampling satisfy this
assumption with
appropriate 
$\mu_n(i)$. 

Given sufficient exploration, 
Proposition \ref{prop:sufficient} recommends a reasonable value to set for the
regularization parameter. In particular setting $\lambda$ to a value
such that $\lambda \le \lambda_{\min}/\log\log n$ occurs with high probability
suffices to ensure that the $\vect{W}$-decorrelated estimate is approximately
unbiased. Correspondingly, the MSE (or equivalently variance)
of the $\vect{W}$-decorrelated estimate need not be smaller than that of the original
$\ols$ estimate. Indeed the variance scales as $1/\lambda$, which exceeds with high probability
the $1/\lambda_{\min}$ scaling for the MSE. This is the cost paid for 
debiasing $\ols$ estimate.

Before we move to the inference results, note
that the procedure requires only access to high probability lower bounds on $\lambda_{\min}$,
which intuitively quantifies the 
exploration of the data collection policy. 
In comparison with methods such as
propensity score weighting or conditional likelihood optimization, this 
represents rather coarse information about the 
data collection process. In particular, given access to propensity scores
or conditional likelihoods
one can simulate the process to extract appropriate values for the regularization
$\lambda(n)$. This is the approach we take in the experiments of Section \ref{sec:expt}. 
Moreover, propensity scores or conditional likelihoods are ineffective when data 
collection policies make adaptive decisions that are deterministic 
given the history. A important example is
that of UCB algorithms
for bandits, which make deterministic choices of arms. 

\begin{algorithm}[t]

    Input: {sample $(y_i, \mathbf{x}_i)_{i\le n}$, regularization $\lambda$, unit vector $\vect{v} \in\reals^p$, confidence level $\alpha \in (0, 1)$, noise estimate $\hat{\sigma}^2$}.

    Compute: $\betahat_\ols = (\vect{X}_n^\sT\vect{X}_n)^{-1} \vect{X}_n \vect{y}_n$. 

    Setting $\vect{W}_0 = 0$, compute $\vect{W}_i = [\vect{W}_{i-1} \vect{w}_i]$ with
    $\vect{w}_i = (\id_p - \vect{W}_{i-1}\vect{X}_i^\sT) \vect{x}_i /(\lambda + \twonorm{\vect{x}_{i}}^2$), for $i = 1, 2, \dots , n$. 

    Compute $\betahat^d = \betahat_\ols + \vect{W}_n (y - \vect{X}_n \betahat_\ols)$
    and $\hat{\sigma}(\vect{v}) = \hat{\sigma}\<\vect{v}, \vect{W}_n\vect{W}^\sT_n \vect{v}\>^{1/2}$

    Output: decorrelated estimate $\betahat^d$ and CI interval $ I(\vect{v}, \alpha)=  [ \<\vect{v}, \betahat^d\> -\hat{\sigma}(\vect{v})\Phi^{-1}(1-\alpha), \<\vect{v}, \betahat^d\>+\hat{\sigma}(\vect{v}) \Phi^{-1}(1-\alpha)]$. 
  \caption{$\vect{W}$-Decorrelation Method \label{alg:decorrmethod}}
\end{algorithm}

\subsection{A central limit theorem and confidence intervals}

Our final result is a central limit theorem that provides an
alternative to the stability condition of Theorem \ref{thm:lai} and
standard martingale CLTs.
{ Standard martingale CLTs~\citep[see, e.g.,][]{lai1982least,dvoretzky1972asymptotic} require convergence of $\sum_i \vect{w}_i\vect{w}_i^\sT /n$ to a constant, but this convergence condition is violated in many examples of interest, including the AR examples in Section~\ref{sec:expt}.} 

Let $(X_{i, n}, \cF_{i, n}, 1\le i \le n )$ be a martingale
difference array, with the associated
sum process $S_n = \sum_{i\le n} X_{i, n}$ and covariance
process $V_{n} = \sum_{i\le n} \E\{X_{i, n}^2 \vert \cF_{i-1, n}\}$.

\begin{assumption}\label{assmp:momentstability}
\begin{enumerate}
 \item Moments are stable: 
for $a=1, 2$, the following
limit holds 
\begin{align}
\lim_{n\to \infty} \E\Big\{\sum_{i\le n} V_{n}^{-a/2} \Big\lvert  \E\{ X_{i, n}^a | \cF_{i-1, n}, V_{n} \}%
- \E \{ X_{i, n}^a |\cF_{i-1, n} \}  
\Big\rvert \Big\} &= 0
\end{align}
\item Martingale differences are small:
\begin{align}
\lim_{n\to \infty } \sum_{i\le n} \E\Big\{ \frac{\abs{X_{i, n}}^3}{V_{n}^{3/2}}\Big\} &= 0, \\
\lim_{n\to \infty} \frac{\max_{i\le n} \E\{X_{i, n}^2 | \cF_{i-1, n}\}}{V_n} &= 0 \text{ in probability.}
\end{align}
 \end{enumerate}  \end{assumption}

\begin{theorem}[Martingale CLT]\label{thm:martingaleclt}
Under Assumption \ref{assmp:momentstability}, the rescaled
process satisfies $S_n/\sqrt{V_n} \convD \normal(0, 1)$, i.e.
the following
holds for any bounded, continuous test function
$\varphi:\reals\to\reals$:
\begin{align}
\lim_{n\to\infty} \E\big\{\varphi \big(S_n/\sqrt{V_n}\big)  \big\} &=
\E\big\{\varphi(\xi) \big\}, 
\end{align}
where $\xi\sim\normal(0, 1)$. 
\end{theorem}

The first part of Assumption \ref{assmp:momentstability} is 
an alternate form of stability. 
It controls the dependence of
the conditional covariance of $S_n$ on the first two conditional moments of the martingale increments $X_{i, n}$. In words, it states that the knowledge of the 
conditional covariance $\sum_i \E\{X_{i, n}^2 |\cF_{i-1, n}\}$ does not change the
first two moments of increments $X_{i, n}$ by an appreciable
amount\footnote{See \cite{hall2014martingale}, Theorem 3.4 for an example of a martingale central limit theorem in this flavor.}. 

With a CLT in hand, one can now assign confidence
intervals in the standard fashion, based on the assumption that 
the bias is negligible. For instance, we have
result on two-sided confidence intervals.

\begin{proposition}
Fix any $\alpha > 0$. Suppose that the data collection process
satisfies the assumptions of Theorems \ref{thm:vardom} and \ref{thm:martingaleclt}. 
Set $\lambda = \lambda(n)$ as in Theorem \ref{thm:vardom}, and let
$\widehat{\sigma}$ be a consistent estimate of $\sigma$ as 
in Theorem \ref{thm:lai}. Define $\vect{Q} = \widehat{\sigma}^2 \vect{W}_n \vect{W}_n^\sT$
and the interval
$I(a, \alpha) =  [\betahat^d_a -  \sqrt{Q_{aa}}\Phi^{-1}(1-\alpha/2), \betahat^d_a + \sqrt{Q_{aa}}\Phi^{-1}(1-\alpha/2)$. Then 
\begin{talign}
\limsup_{n\to\infty} \, \P\{\beta_a \not\in  I(a, \alpha) \} &\le \alpha.
\end{talign}
\end{proposition}

\subsection{Stability for multi-armed bandits}

Limited information central limit theorems such
as Theorem \ref{thm:martingaleclt} (or \cite[Theorem 3.4]
{hall2014martingale}), while providing insight into
the problem of determining asymptotics, have assumptions
that are often difficult to check in practice. Therefore, 
sufficient conditions such as the stability assumed in
Theorem \ref{thm:lai} are often preferred while analyzing
the asymptotic behavior of martingales. In this section we circumvent this problem
by proving the standard version of stability (as
assumed in Theorem \ref{thm:lai}) for $\vect{W}$-estimates, 
assuming the matrices $\vect{x}_i \vect{x}_i^\sT$ commute.
While this is not a complete resolution to the problems
posed by limited information martingale CLT's, it applies
to important special cases like multi-armed bandits. 

\begin{figure}[t]
\includegraphics[width=0.31\linewidth]{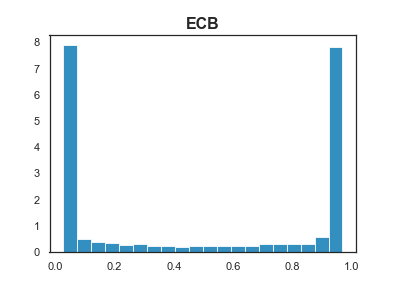} 
\includegraphics[width=0.31\linewidth]{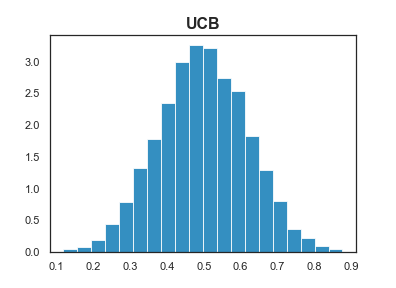}
\includegraphics[width=0.31\linewidth]{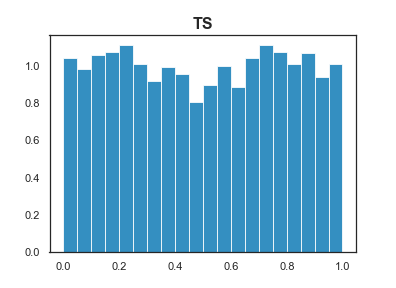}
\caption{Histograms of the distribution of $N_1(n)/n$, the fraction of times arm 1 is picked under $\eps$-greedy, UCB and Thompson sampling.
The bandit problem has $p=2$ arms which have i.i.d. $\Unif([-0.7, 1.3])$
rewards and a time horizon of $n=1000$. The distribution is plotted
over 4000 Monte Carlo iterations. 
\label{fig:mabarmdist}}
\end{figure}

Recall that the stability assumed in Theorem \ref{thm:lai} requires
a non-random sequence of matrices $\vect{A}_n$ so that
\begin{align}
\vect{A}^{-1}_n \bX_n\bX_n^\sT &\stackrel{p}{\to} \id_p
\end{align}
When the vectors $\bx_i$ take values among $\{\bv_1, \dots \bv_p\}$,
a set of orthogonal vectors, we have
\begin{align}
\bX_n\bX_n^\sT &= \sum_i \bx_i\bx_i^\sT \\
&= \sum_{a=1}^p \bv_a\bv_a^\sT \sum_i \ind(\text{arm } a \text{ chosen at time } i), \\
&= \sum_{a=1}^p \bv_a\bv_a^\sT N_a(n),
\end{align}
where we define $N_a(n) = \sum_{i=1}^n \ind(\bx_i = \bv_a)$. Therefore, 
if there existed $\vect{A}_n$ so that the stability condition held, 
then we would have, for each $a$, that $N_a(n)\<\bv_a, \vect{A}^{-1}_n\bv_a\> \to 1$ in probability.

We test this assumption in a simple, but illuminating
setting: a multi-armed bandit problem with $p=2$ 
arms that are \emph{statistically identical}: they
each yield i.i.d. $\Unif([-0.7, 1.3])$ rewards. We run
$\eps$-greedy (with a fixed value  $\eps=0.1$), 
Thompson sampling and a variant of UCB for a time
horizon of $n=1000$ for 4000 Monte Carlo iterations. The
resulting histograms of the fraction $N_1(n)/n$ of times
arm 1 was picked by each of the three policies is given 
in Figure \ref{fig:mabarmdist}. Since the arms are statistically identical,
the algorithm behavior is exchangeable with respect to
switching the arm labels, viz. switching arm 1 for arm 2. In particular, the distribution of $N_1(n)$ and $N_2(n)$ is identical, for a
given policy. Combining this with $N_1(n) + N_2(n) = n$, we 
have that $\E\{N_1(n)\} = \E\{N_2(n)\} = n/2$. Therefore, 
if stability a la Theorem \ref{thm:lai} held, this would imply
that the distribution of fraction $N_1(n)/n$ would be close
to a Dirac delta at $1/2$. 
However, we see that for all the three policies UCB, Thompson sampling and $\eps$-greedy, this is not the case. 
 Indeed, $N_1(n)/n$ has significant variance about $1/2$ for
 all the policies; to wit, the $\eps$-greedy indeed shows a sharp bimodal behavior. Consequently, the stability condition required
 by Theorem \ref{thm:lai} \emph{fails to hold} quite dramatically
 in this simple setting. As we observe in Section \ref{sec:expt}, this affects significantly the limiting distribution of the sample means,
 which have non-trivial bias and poor coverage of nominal confidence intervals.

In the following, we will prove that $\vect{W}$-estimates 
are indeed stable in the sense of Theorem \ref{thm:lai},
given a judicious choice of $\lambda = \lambda(n)$.  
Suppose that for each time $i$, $\vect{x}_i \in \{\bv_1, \dots, \bv_p\}$
the latter being a set of orthogonal (not necessarily unit normed) 
vectors $\bv_a$. We also define $N_a(i) =  \sum_{j\le i} \ind(\bx_j = \bv_a) $.
The following proposition shows that when $\lambda = \lambda(n)$
is set appropriately, the $\vect{W}$-estimate is stable.  
\begin{proposition}\label{prop:commvarianceisstable}
Suppose that the sequence $\lambda = \lambda(n)$ 
satisfies  $(i)$ $\lambda(n)/\lambda_{\min}(\bX_n\bX_n^\sT)\to 0$ in probability and $(ii)$  $\lambda(n) \to \infty$. 
Then the following holds: 
\begin{align}
 \lambda (n)\vect{W_n} \vect{W_n} ^\sT&\stackrel{L_1}{\to} \frac{\id_p}{2}.  
\end{align}
\end{proposition}

Along with Theorem \ref{thm:vardom} and
Proposition \ref{prop:sufficient}, this immediately 
yields a simple corollary on the distribution of $\vect{W}$-estimates in the
commutative setting. The key advantage here is that 
we are able to circumvent the assumptions of the limited information
central limit Theorem \ref{thm:martingaleclt}. 
\begin{corollary}
Suppose that $\vect{x}_i$ take values in $\{\bv_1, \dots \bv_p\}$, 
a set of orthogonal vectors. Let $\widehat{\sigma}^2$
be an estimate of the variance $\sigma^2$ as obtained from
Theorem \ref{thm:lai} and $\betahat^d$ be the $\vect{W}$-estimate
obtained using $\lambda = \lambda(n)$ so that 
 $\lambda(n)\log\log(n) \E\{\lambda_{\min}^{-1}(\vect{X}_n^\sT\vect{X}_n)\} \to 0$. 
 Then, with $\xi\sim\normal(0, \id_p)$ and any Borel set $A \subseteq \reals^p$:
 \begin{align}
 \lim_{n\to\infty} \P\Big\{ (\widehat{\sigma}^2\lambda(n) \vect{W}_n\vect{W}_n^\sT)^{-1/2}(\betahat^d - \beta ) \in A \Big\} &= \P\{\xi \in A\}. 
 \end{align}
\end{corollary}

\section{Related work} \label{sec:related}

There is extensive work in statistics and 
econometrics on stochastic regression
models \citep{wei1985asymptotic,lai1994asymptotic,chen1999strong,heyde2008quasi} and non-stationary
time series \citep{shumway2006time,enders2008applied,phillips1988testing}. 
This line of work is analogous to Theorem \ref{thm:lai} or 
restricted to specific time series models. We instead focus on literature from sequential decision-making,  
policy learning and causal inference that closely resembles our work in terms of
goals, techniques and applicability. 

The seminal work of Lai and Robbins \citep{robbins1985some,lai1985asymptotically} has
spurred a vast literature on multi-armed bandit problems and sequential experiments that propose allocation algorithms based on confidence bounds (see \cite{bubeck2012regret}
and references therein). A variety of confidence bounds and corresponding rules have been
proposed \citep{auer2002using,dani2008stochastic,rusmevichientong2010linearly,abbasi2011online,jamieson2014lil}
based on martingale concentration and the law of iterated logarithm. While these results
can certainly be used to compute valid confidence intervals, they
are conservative for
a few reasons. First, they do not explicitly account for bias in $\ols$ estimates and, correspondingly,
must be wider to account for it.  Second, obtaining optimal constants in the concentration inequalities can require sophisticated tools even for non-adaptive
data \citep{Led96,ledoux2005concentration}. This is evidenced in all of our
experiments which show that concentration inequalities yield valid, but conservative
intervals.  

A closely-related line of work is that of learning from logged data \citep{li2011unbiased,dudik2011doubly,swaminathan2015batch} 
and policy learning \citep{athey2017efficient,kallus2017balanced}. The focus here is efficiently estimating the 
reward (or value) of a certain test policy using data collected from a different policy. For linear
models, this reduces to accurate prediction which is directly
related to the estimation error on the parameters $\beta$. While our work shares some features, we focus on unbiased estimation of the 
parameters and obtaining accurate confidence intervals for linear functions of the parameters. Some of the work on learning from logged data also 
builds on propensity scores
and their estimation  \citep{imbens2000role,lunceford2004stratification}. 

\citet{villar2015multi} empirically demonstrate the presence
of bias for a number of multi-armed bandit algorithms. Recent work by \citet{dimakopoulou2017estimation}
also shows a similar effect in contextual bandits. Along with a result on the sign of the bias,
\cite{nie2017why} also propose conditional likelihood optimization methods to estimate parameters of the linear 
model. Through the lens of selective inference, they also propose methods to randomize the data
collection process that simultaneously lower bias and reduce the MSE. Their techniques rely
on considerable information about (and control over) the data generating process, in particular
the probabilities of choosing a specific action at each point in the data selection. This can be
viewed as lying on the opposite end of the spectrum from our work, which attempts to use only the
data at hand, along with coarse aggregate information on the exploration inherent in the data generating process.  It is an interesting, and open, direction to consider approaches that
can combine the strengths of our approach and that of \cite{nie2017why}.

\section{Experiments} \label{sec:expt}

In this section we empirically validate the decorrelated
estimators in
two scenarios that involve sequential dependence
in covariates. 
Our first scenario is a simple
experiment of multi-armed bandits 
while the second scenario is autoregressive time series data.
In these cases, we compare the empirical coverage and typical
widths of confidence intervals for parameters obtained
via three methods: $(i)$ classical OLS theory, 
$(ii)$ concentration inequalities and $(iii)$ decorrelated
estimates. \iftoggle{arxiv}
{Code for reproducing our experiments are
available \citep{deshpande2018decorrelating}.}{}

\subsection{Multi-armed bandits}

\begin{figure}[!bp]
\centering
\includegraphics[width=.49\linewidth]{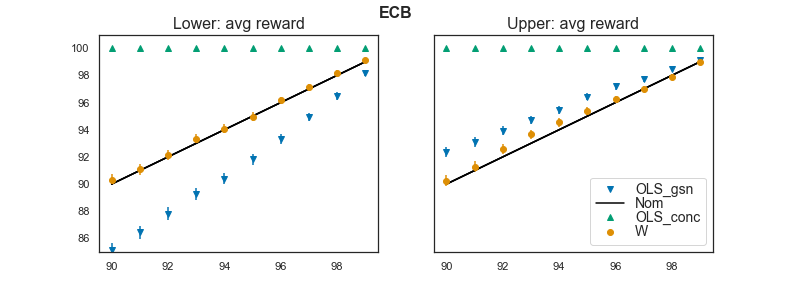}
\includegraphics[width=.49\linewidth]{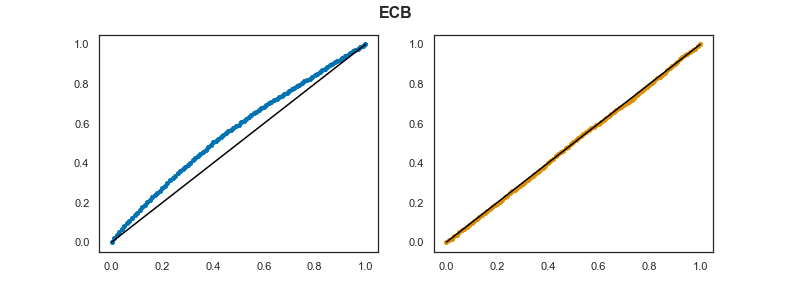}

\includegraphics[width=.49\linewidth]{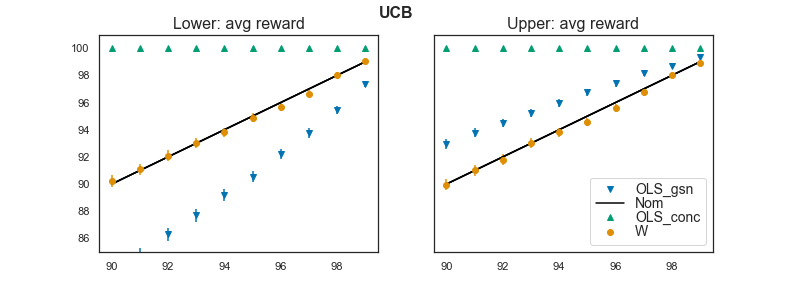}
\includegraphics[width=.49\linewidth]{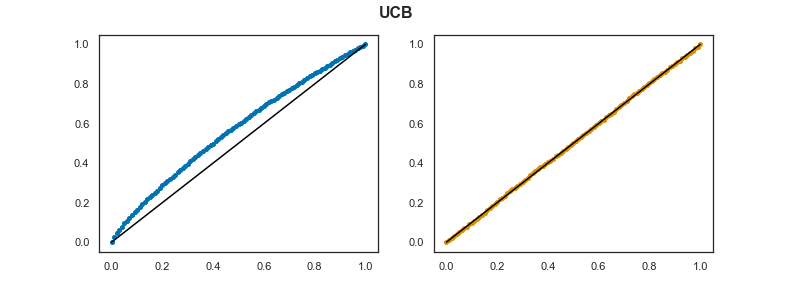}

\includegraphics[width=.49\linewidth]{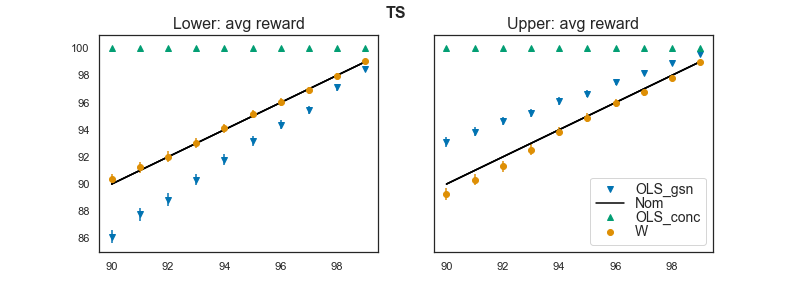}
\includegraphics[width=.49\linewidth]{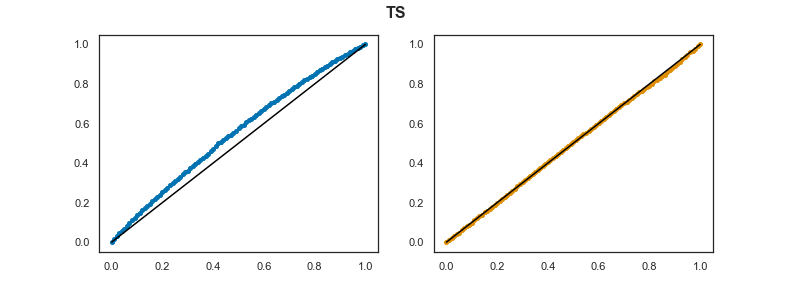}

	\caption{Multi-armed bandit results. Left: One-sided confidence region coverage for $\ols$ and decorrelated $\vect{W}$-decorrelated estimates of the average
	reward $0.5\beta_1 + 0.5\beta_2$.  Right: Probability (PP) plots for the $\ols$ and $\vect{W}$-decorrelated estimate errors of the average reward.} \label{fig:mab}
\end{figure}
\begin{figure*}[!tp]
\centering
\includegraphics[width=0.31\linewidth]{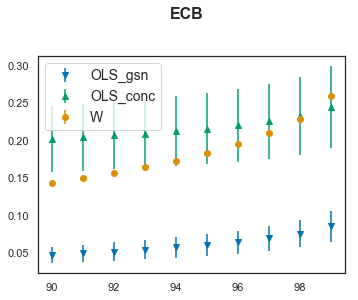} 
\includegraphics[width=0.31\linewidth]{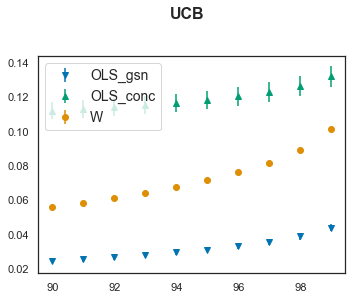}
\includegraphics[width=0.31\linewidth]{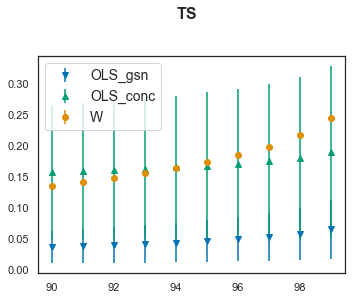}
\caption{Multi-armed bandit results. Mean 2-sided confidence interval widths (error bars show 1 standard deviation) for the 
average reward $0.5\beta_1+ 0.5\beta_2$ in the MAB experiment. \label{fig:mabwidth}}
\end{figure*}

In this section, we demonstrate the utility of the
$\vect{W}$-estimator for a stochastic multi-armed
bandit setting.  
\citet{villar2015multi} studied this problem in the context of patient allocation in clinical trials. Here the trial
proceeds in a sequential fashion with the $i^\text{th}$
patient given one of $p$ treatments, encoded as $\vect{x}_i =\vect{e}_a$ with $a\in [p]$, and $y_i$ denoting the outcome
observed. We model the outcome as $y_i = \<\vect{x}_i, \beta\> + \eps_i$ where $\eps_i\sim\Unif([-1, 1])$ with $\beta = (0.3, 0.3)$ being
the mean outcome of the treatments. Note that
the two treatments are \emph{statistically identical} in terms
of outcome. As we will see, the adaptive sampling induced by the bandit
strategies, however, introduces significant biases in the 
estimates. 

We sequentially assign one of $p=2$ treatments to each of $n=1000$ patients using one of three policies (i) an $\eps$-greedy policy (called ECB or Epsilon Current Belief), $(ii)$ a
practical UCB strategy based on the law of iterated logarithm (UCB) \citep{jamieson2014lil} and (iii) Thompson sampling \citep{thompson1933likelihood}. The ECB and TS sampling strategies
are Bayesian. They place an independent Gaussian prior (with mean $\mu_0=0.3$ and variance $\sigma^2_0=0.33$) on each unknown mean outcome parameter and form an updated posterior belief concerning $\beta$ following each treatment administration $\vect{x_i}$ and observation $y_i$. 

For ECB, the treatment administered to patient $i$ is, with probability $1-\eps = .9$, the treatment with the largest posterior mean; with probability $1-\eps$, a uniformly random treatment is administered instead, to ensure sufficient exploration of all treatments.  Note that this strategy satisfies condition \cref{eq:explorecond1} with $\mu_n(i) = \eps/p$. For TS, at each patient $i$, a sample $\betahat$ of the mean treatment effect is drawn from the posterior belief. The treatment assigned to patient is the one maximizing the sampled mean treatment, i.e. $a_*(i) = \arg\max_{a\in [p]} \betahat_a$. In UCB, the algorithm maintains a
score for each arm $a\in [p]$ that is a combination of the mean reward that the arm
achieves and the empirical uncertainty of the reward. For each patient $i$, the UCB
algorithm chooses the arm maximizing this score, and updates the score according 
to a fixed rule. For details on the specific implementation, see \citet{jamieson2014lil}. 
Our goal is to produce confidence intervals for the 
$\beta_a$ of each treatment based on the data 
adaptively collected from these standard bandit algorithms. 
We will compare the estimates and corresponding intervals
for the \emph{average reward} $0.5\beta_1 + 0.5\beta_2$. 
As the two arms/treatments are statistically identical,
this isolates the effect of adaptive sampling on 
the obtained estimates.

We repeat the simulation for $5000$ Monte Carlo runs. 
From each trial, we estimate the parameters $\beta$ using both $\ols$ and the $\vect{W}$-estimator with $\lambda = \hat{\lambda}_{5\%, \pi}$ which is the $5^\text{th}$ percentile of $\lambda_{\min}(n)$ achieved by the policy $\pi \in \{\text{ECB}, \text{UCB}, \text{TS}\}$. This choice is guided by Corollary \ref{thm:vardom}.
 
We compare the quality of confidence regions for the 
average  reward $0.5\beta_1 + 0.5\beta_2$ obtained
from the $\vect{W}$-decorrelated estimator, the $\ols$ estimator with
standard Gaussian theory ($\ols_{\text{gsn}}$), and the $\ols$ 
estimator using concentration inequalities ($\ols_{\text{conc}}$) \citep[Sec. 4]{abbasi2011online}.
Figure \ref{fig:mab} (left column) shows that the $\ols$ Gaussian have have inconsistent coverage from the nominal. This is consistent with
the observation that the sample means are biased negatively 
\citep{nie2017why}. The concentration OLS tail bounds are all 
conservative, producing nearly 100\% coverage, irrespective of the nominal level. This is intuitive, since they must
account for the bias in sample means \citep{nie2017why}. Meanwhile, the decorrelated intervals improves coverage uniformly over $\ols$ 
intervals, often achieving the nominal coverage. 

Figure \ref{fig:mab} (right column) shows the PP plots of $\ols$ and $\vect{W}$-estimator errors for the average reward $0.5\beta_1 +0.5 \beta_2$. Recall that a PP plot between two distributions on the real line with densities $P$ and $Q$ is the parametric curve $(P(z), Q(z)), z\in \reals$ \cite[Chapter 4.7]{gibbons2011nonparametric}.
The distribution of $\ols$ errors is clearly seen to be
distinctly non-Gaussian. 

Figure \ref{fig:mabwidth} summarizes the distribution of $2$-sided interval widths produced by each method for the sum reward.
As expected, the $\vect{W}$-decorrelation intervals are wider than those of $\ols_{\text{gsn}}$ but 
compare favorably with those provided by $\ols_{\text{conc}}$.
For UCB, the mean $\ols_{\text{conc}}$ widths 
are always largest.
For TS and ECB, $\vect{W}$-decorrelation yields smaller intervals than $\ols_{\text{conc}}$ for moderate confidence levels
and comparable for high confidence levels. From this, we see that $\vect{W}$-decorrelation intervals can be 
considerably less conservative than the concentration-based confidence intervals.

\subsection{Autoregressive time series}
\ifboolexpr{togl{arxiv} or togl{neuripscausal}}{

\begin{figure}[t]
\begin{tabular*}{0.5\linewidth}{c}
\includegraphics[width=0.45\linewidth]{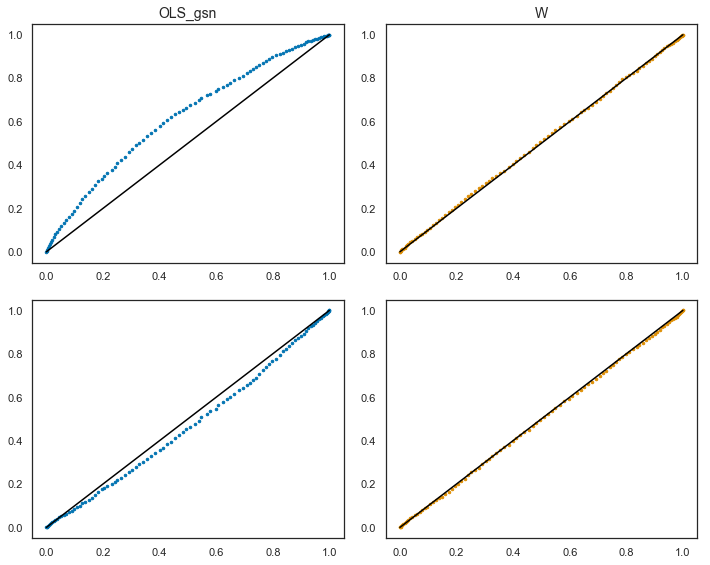}
\end{tabular*}
\hfill
\begin{tabular*}{0.5\linewidth}{c}
\includegraphics[width=0.45\linewidth]{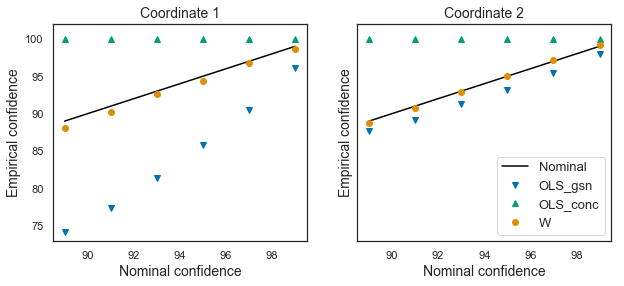}\\
\includegraphics[width=0.45\linewidth]{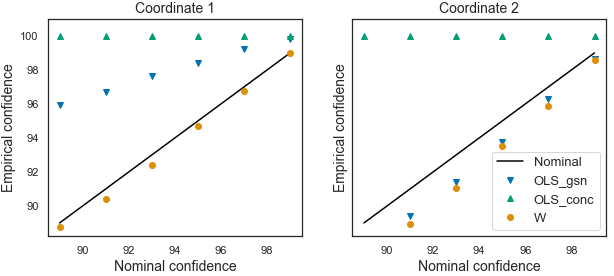}
\end{tabular*}
\begin{center}
\includegraphics[width=0.6\textwidth]{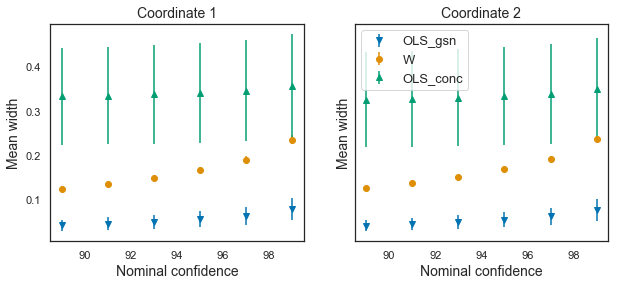}
\end{center}
	\caption{ AR(2) time series results.
	Upper left: PP plot for the distribution of errors of standard $\ols$ estimate and the $W$-decorrelated estimate. Upper right: Lower (top) and upper (bottom) coverage probabilities for $\ols$ with Gaussian intervals, $\ols$ with concentration inequality intervals, and decorrelated $\vect{W}$-decorrelated estimate intervals.  Note that `Conc' has always 100\% coverage. 
	Bottom: Average 2 sided confidence interval widths obtained using the $\ols$ estimator with standard Gaussian
	theory, $\ols$ with concentration inequalities and the
	$\vect{W}$-decorrelated estimator. 
} \label{fig:ar2}
\end{figure}
}{}

In this section, we consider
the classical AR$(p)$ model
where $y_i = \sum_{\ell \le p} \beta_\ell y_{i - \ell} + \eps_i.
$. 
We generate data for the
model \label{eq:timeseriesmodel} with parameters $p = 2, n = 50, \beta = (0.95, 0.2)$, 
$y_0 = 0$ and $\eps_i \sim\Unif([-1, 1])$; all estimates
are computed over
$4000$ monte carlo iterations. 

We plot the coverage confidences for various values of 
the nominal
on the right panel of Figure \ref{fig:ar2}. The PP plot of the error
distributions on the bottom right panel of Figure \ref{fig:ar2} shows that 
the $\ols$ errors are skewed downwards, while the $\vect{W}$-estimate errors are 
nearly Gaussian. We obtain the following improvements over the comparison methods of
$\ols$ standard errors $\ols_{\text{gsn}}$ and concentration inequality widths $\ols_{\text{conc}}$
\citep{abbasi2011online}

 The Gaussian $\ols$ confidence regions systematically give incorrect empirical coverage. Meanwhile, the concentration inequalities provide very conservative intervals, with nearly 100\% coverage, irrespective of the nominal level.  
In contrast, our decorrelated intervals achieve empirical coverage that closely approximates the nominal levels. 
These coverage improvements are enabled by an increase in width over that of $\ols_{\text{gsn}}$, but the $\vect{W}$-estimate widths are systematically smaller than those of the concentration inequalities.

\iftoggle{icml}{
\begin{figure}[!htb]
\centering
\includegraphics[width=0.8\linewidth]{{"Nonstationary_AR2_Empirical_Coverage_Probabilities_Lower_Tail"}.png}
\includegraphics[width=0.8\linewidth]{{"Nonstationary_AR2_Empirical_Coverage_Probabilities_Upper_Tail"}.png}
\includegraphics[width=0.83\linewidth]{{"Nonstationary_AR2_Width_Comparison"}.png}
\includegraphics[width=0.83\linewidth]{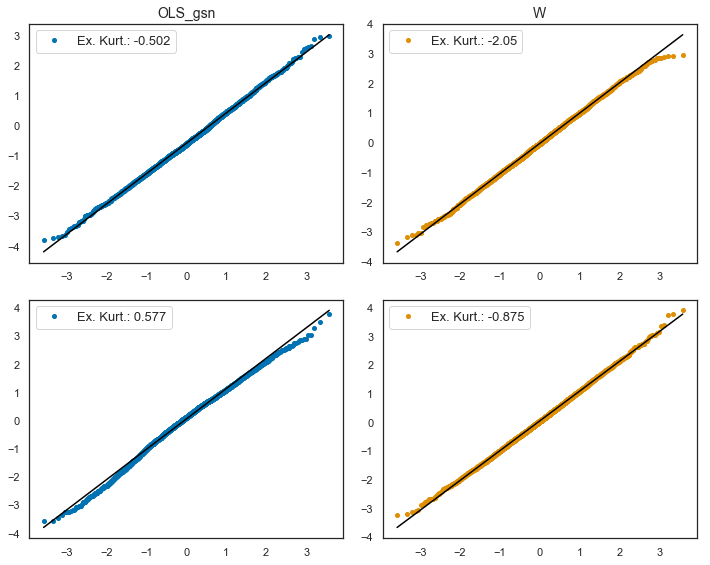}
	\caption{
Lower (Top left) and upper (Top right) coverage probabilities for $\ols$ with Gaussian intervals, $\ols$ with concentration inequality intervals and  $\vect{W}$-decorrelated estimate intervals. 
	QQ plot with kurtosis inset (bottom right) errors in $\ols$ estimate and $W$-decorrelated estimate. Mean confidence widths (bottom left) for $\ols$, concentration and $\vect{W}$-decorrelated estimates. Error
	bars show one standard deviation.} \label{fig:ar1}
\end{figure}}{}

\subsection*{Acknowledgements}
The authors would like to thank Adel Javanmard and Lucas Janson for feedback on an
earlier version of this paper.

\ifboolexpr{togl{neuripscausal}}{
\bibliographystyle{plainnat}
\bibliography{all-bibliography}
}{}

\ifboolexpr{togl{arxiv}}{
\bibliographystyle{plainnat}
}{}

\iftoggle{icml}{
  \bibliographystyle{icml2018}
}
{
\iftoggle{icml}{\newpage}{}
\bibliography{all-bibliography}}

\begin{thebibliography}{49}
\providecommand{\natexlab}[1]{#1}
\providecommand{\url}[1]{\texttt{#1}}
\expandafter\ifx\csname urlstyle\endcsname\relax
  \providecommand{\doi}[1]{doi: #1}\else
  \providecommand{\doi}{doi: \begingroup \urlstyle{rm}\Url}\fi

\bibitem[Abbasi-Yadkori et~al.(2011)Abbasi-Yadkori, P{\'a}l, and
  Szepesv{\'a}ri]{abbasi2011online}
Yasin Abbasi-Yadkori, D{\'a}vid P{\'a}l, and Csaba Szepesv{\'a}ri.
\newblock Online least squares estimation with self-normalized processes: An
  application to bandit problems.
\newblock \emph{arXiv preprint arXiv:1102.2670}, 2011.

\bibitem[Athey and Wager(2017)]{athey2017efficient}
Susan Athey and Stefan Wager.
\newblock Efficient policy learning.
\newblock \emph{arXiv preprint arXiv:1702.02896}, 2017.

\bibitem[Audibert and Bubeck(2009)]{audibert2009minimax}
Jean-Yves Audibert and S{\'e}bastien Bubeck.
\newblock Minimax policies for adversarial and stochastic bandits.
\newblock In \emph{COLT}, pages 217--226, 2009.

\bibitem[Auer(2002)]{auer2002using}
Peter Auer.
\newblock Using confidence bounds for exploitation-exploration trade-offs.
\newblock \emph{Journal of Machine Learning Research}, 3\penalty0
  (Nov):\penalty0 397--422, 2002.

\bibitem[Bertsekas(2011)]{bertsekas2011incremental}
Dimitri~P Bertsekas.
\newblock Incremental proximal methods for large scale convex optimization.
\newblock \emph{Mathematical programming}, 129\penalty0 (2):\penalty0 163,
  2011.

\bibitem[Bubeck et~al.(2012)Bubeck, Cesa-Bianchi, et~al.]{bubeck2012regret}
S{\'e}bastien Bubeck, Nicolo Cesa-Bianchi, et~al.
\newblock Regret analysis of stochastic and nonstochastic multi-armed bandit
  problems.
\newblock \emph{Foundations and Trends{\textregistered} in Machine Learning},
  5\penalty0 (1):\penalty0 1--122, 2012.

\bibitem[Castro and Nowak(2008)]{castro2008minimax}
Rui~M Castro and Robert~D Nowak.
\newblock Minimax bounds for active learning.
\newblock \emph{IEEE Transactions on Information Theory}, 54\penalty0
  (5):\penalty0 2339--2353, 2008.

\bibitem[Chan and Wei(1987)]{chan1987asymptotic}
Ngai~H Chan and Ching-Zong Wei.
\newblock Asymptotic inference for nearly nonstationary ar (1) processes.
\newblock \emph{The Annals of Statistics}, pages 1050--1063, 1987.

\bibitem[Chen et~al.(1999)Chen, Hu, Ying, et~al.]{chen1999strong}
Kani Chen, Inchi Hu, Zhiliang Ying, et~al.
\newblock Strong consistency of maximum quasi-likelihood estimators in
  generalized linear models with fixed and adaptive designs.
\newblock \emph{The Annals of Statistics}, 27\penalty0 (4):\penalty0
  1155--1163, 1999.

\bibitem[Dani et~al.(2008)Dani, Hayes, and Kakade]{dani2008stochastic}
Varsha Dani, Thomas~P Hayes, and Sham~M Kakade.
\newblock Stochastic linear optimization under bandit feedback.
\newblock In \emph{COLT}, pages 355--366, 2008.

\bibitem[Deshpande and Montanari(2012)]{deshpande2012linear}
Yash Deshpande and Andrea Montanari.
\newblock Linear bandits in high dimension and recommendation systems.
\newblock In \emph{Communication, Control, and Computing (Allerton), 2012 50th
  Annual Allerton Conference on}, pages 1750--1754. IEEE, 2012.

\bibitem[Deshpande et~al.(2018)Deshpande, Mackey, Syrgkanis, and
  Taddy]{deshpande2018decorrelating}
Yash Deshpande, Lester Mackey, Vasilis Syrgkanis, and Matt Taddy.
\newblock Accurate inference for adaptive linear models, 2018.
\newblock URL \url{https://bit.ly/2mJxULg}.

\bibitem[Dimakopoulou et~al.(2017)Dimakopoulou, Athey, and
  Imbens]{dimakopoulou2017estimation}
Maria Dimakopoulou, Susan Athey, and Guido Imbens.
\newblock Estimation considerations in contextual bandits.
\newblock \emph{arXiv preprint arXiv:1711.07077}, 2017.

\bibitem[Dud{\'\i}k et~al.(2011)Dud{\'\i}k, Langford, and Li]{dudik2011doubly}
Miroslav Dud{\'\i}k, John Langford, and Lihong Li.
\newblock Doubly robust policy evaluation and learning.
\newblock \emph{arXiv preprint arXiv:1103.4601}, 2011.

\bibitem[Dvoretzky(1972)]{dvoretzky1972asymptotic}
Aryeh Dvoretzky.
\newblock Asymptotic normality for sums of dependent random variables.
\newblock In \emph{Proc. 6th Berkeley Symp. Math. Statist. Probab}, volume~2,
  pages 513--535, 1972.

\bibitem[Enders(2008)]{enders2008applied}
Walter Enders.
\newblock \emph{Applied econometric time series}.
\newblock John Wiley \& Sons, 2008.

\bibitem[Garivier and Capp{\'e}(2011)]{garivier2011kl}
Aur{\'e}lien Garivier and Olivier Capp{\'e}.
\newblock The kl-ucb algorithm for bounded stochastic bandits and beyond.
\newblock In \emph{Proceedings of the 24th annual Conference On Learning
  Theory}, pages 359--376, 2011.

\bibitem[Gibbons and Chakraborti(2011)]{gibbons2011nonparametric}
Jean~Dickinson Gibbons and Subhabrata Chakraborti.
\newblock \emph{Nonparametric statistical inference}.
\newblock Springer, 2011.

\bibitem[Hall and Heyde(2014)]{hall2014martingale}
Peter Hall and Christopher~C Heyde.
\newblock \emph{Martingale limit theory and its application}.
\newblock Academic press, 2014.

\bibitem[Heyde(2008)]{heyde2008quasi}
Christopher~C Heyde.
\newblock \emph{Quasi-likelihood and its application: a general approach to
  optimal parameter estimation}.
\newblock Springer Science \& Business Media, 2008.

\bibitem[Imbens(2000)]{imbens2000role}
Guido~W Imbens.
\newblock The role of the propensity score in estimating dose-response
  functions.
\newblock \emph{Biometrika}, 87\penalty0 (3):\penalty0 706--710, 2000.

\bibitem[Jamieson et~al.(2014)Jamieson, Malloy, Nowak, and
  Bubeck]{jamieson2014lil}
Kevin Jamieson, Matthew Malloy, Robert Nowak, and S{\'e}bastien Bubeck.
\newblock lil’ucb: An optimal exploration algorithm for multi-armed bandits.
\newblock In \emph{Conference on Learning Theory}, pages 423--439, 2014.

\bibitem[Javanmard and Montanari(2014{\natexlab{a}})]{javanmard2014confidence}
Adel Javanmard and Andrea Montanari.
\newblock Confidence intervals and hypothesis testing for high-dimensional
  regression.
\newblock \emph{Journal of Machine Learning Research}, 15\penalty0
  (1):\penalty0 2869--2909, 2014{\natexlab{a}}.

\bibitem[Javanmard and Montanari(2014{\natexlab{b}})]{javanmard2014hypothesis}
Adel Javanmard and Andrea Montanari.
\newblock Hypothesis testing in high-dimensional regression under the gaussian
  random design model: Asymptotic theory.
\newblock \emph{IEEE Transactions on Information Theory}, 60\penalty0
  (10):\penalty0 6522--6554, 2014{\natexlab{b}}.

\bibitem[Kallus(2017)]{kallus2017balanced}
Nathan Kallus.
\newblock Balanced policy evaluation and learning.
\newblock \emph{arXiv preprint arXiv:1705.07384}, 2017.

\bibitem[Kulis and Bartlett(2010)]{kulis2010implicit}
Brian Kulis and Peter~L Bartlett.
\newblock Implicit online learning.
\newblock In \emph{Proceedings of the 27th International Conference on Machine
  Learning (ICML-10)}, pages 575--582, 2010.

\bibitem[Lai and Siegmund(1983)]{lai1983fixed}
Tse—Leung Lai and David Siegmund.
\newblock Fixed accuracy estimation of an autoregressive parameter.
\newblock \emph{The Annals of Statistics}, pages 478--485, 1983.

\bibitem[Lai(1994)]{lai1994asymptotic}
Tze~Leung Lai.
\newblock Asymptotic properties of nonlinear least squares estimates in
  stochastic regression models.
\newblock \emph{The Annals of Statistics}, pages 1917--1930, 1994.

\bibitem[Lai and Robbins(1985)]{lai1985asymptotically}
Tze~Leung Lai and Herbert Robbins.
\newblock Asymptotically efficient adaptive allocation rules.
\newblock \emph{Advances in applied mathematics}, 6\penalty0 (1):\penalty0
  4--22, 1985.

\bibitem[Lai and Wei(1982)]{lai1982least}
Tze~Leung Lai and Ching~Zong Wei.
\newblock Least squares estimates in stochastic regression models with
  applications to identification and control of dynamic systems.
\newblock \emph{The Annals of Statistics}, pages 154--166, 1982.

\bibitem[Ledoux(1996)]{Led96}
M.~Ledoux.
\newblock \emph{{Isoperimetry and Gaussian analysis}}, volume 1648.
\newblock Springer, Providence, 1996.

\bibitem[Ledoux(2005)]{ledoux2005concentration}
Michel Ledoux.
\newblock \emph{The concentration of measure phenomenon}.
\newblock Number~89. American Mathematical Soc., 2005.

\bibitem[Li et~al.(2010)Li, Chu, Langford, and Schapire]{li2010contextual}
Lihong Li, Wei Chu, John Langford, and Robert~E Schapire.
\newblock A contextual-bandit approach to personalized news article
  recommendation.
\newblock In \emph{Proceedings of the 19th international conference on World
  wide web}, pages 661--670. ACM, 2010.

\bibitem[Li et~al.(2011)Li, Chu, Langford, and Wang]{li2011unbiased}
Lihong Li, Wei Chu, John Langford, and Xuanhui Wang.
\newblock Unbiased offline evaluation of contextual-bandit-based news article
  recommendation algorithms.
\newblock In \emph{Proceedings of the fourth ACM international conference on
  Web search and data mining}, pages 297--306. ACM, 2011.

\bibitem[Lunceford and Davidian(2004)]{lunceford2004stratification}
Jared~K Lunceford and Marie Davidian.
\newblock Stratification and weighting via the propensity score in estimation
  of causal treatment effects: a comparative study.
\newblock \emph{Statistics in medicine}, 23\penalty0 (19):\penalty0 2937--2960,
  2004.

\bibitem[Nagumo and Noda(1967)]{nagumo1967learning}
Jin-Ichi Nagumo and Atsuhiko Noda.
\newblock A learning method for system identification.
\newblock \emph{IEEE Transactions on Automatic Control}, 12\penalty0
  (3):\penalty0 282--287, 1967.

\bibitem[Nie et~al.(2017)Nie, Xiaoying, Taylor, and Zou]{nie2017why}
Xinkun Nie, Tian Xiaoying, Jonathan Taylor, and James Zou.
\newblock Why adaptively collected data have negative bias and how to correct
  for it.
\newblock 2017.

\bibitem[Phillips and Perron(1988)]{phillips1988testing}
Peter~CB Phillips and Pierre Perron.
\newblock Testing for a unit root in time series regression.
\newblock \emph{Biometrika}, 75\penalty0 (2):\penalty0 335--346, 1988.

\bibitem[Robbins(1985)]{robbins1985some}
Herbert Robbins.
\newblock Some aspects of the sequential design of experiments.
\newblock In \emph{Herbert Robbins Selected Papers}, pages 169--177. Springer,
  1985.

\bibitem[Rusmevichientong and Tsitsiklis(2010)]{rusmevichientong2010linearly}
Paat Rusmevichientong and John~N Tsitsiklis.
\newblock Linearly parameterized bandits.
\newblock \emph{Mathematics of Operations Research}, 35\penalty0 (2):\penalty0
  395--411, 2010.

\bibitem[Russo(2016)]{russo2016simple}
Daniel Russo.
\newblock Simple bayesian algorithms for best arm identification.
\newblock In \emph{Conference on Learning Theory}, pages 1417--1418, 2016.

\bibitem[Shumway and Stoffer(2006)]{shumway2006time}
Robert~H Shumway and David~S Stoffer.
\newblock \emph{Time series analysis and its applications: with R examples}.
\newblock Springer Science \& Business Media, 2006.

\bibitem[Swaminathan and Joachims(2015)]{swaminathan2015batch}
Adith Swaminathan and Thorsten Joachims.
\newblock Batch learning from logged bandit feedback through counterfactual
  risk minimization.
\newblock \emph{Journal of Machine Learning Research}, 16:\penalty0 1731--1755,
  2015.

\bibitem[Thompson(1933)]{thompson1933likelihood}
William~R Thompson.
\newblock On the likelihood that one unknown probability exceeds another in
  view of the evidence of two samples.
\newblock \emph{Biometrika}, 25\penalty0 (3/4):\penalty0 285--294, 1933.

\bibitem[Tropp(2012)]{tropp2012user}
Joel~A Tropp.
\newblock User-friendly tail bounds for sums of random matrices.
\newblock \emph{Foundations of computational mathematics}, 12\penalty0
  (4):\penalty0 389--434, 2012.

\bibitem[Van~de Geer et~al.(2014)Van~de Geer, B{\"u}hlmann, Ritov, Dezeure,
  et~al.]{van2014asymptotically}
Sara Van~de Geer, Peter B{\"u}hlmann, Ya’acov Ritov, Ruben Dezeure, et~al.
\newblock On asymptotically optimal confidence regions and tests for
  high-dimensional models.
\newblock \emph{The Annals of Statistics}, 42\penalty0 (3):\penalty0
  1166--1202, 2014.

\bibitem[Villar et~al.(2015)Villar, Bowden, and Wason]{villar2015multi}
Sofia Villar, Jack Bowden, and James Wason.
\newblock Multi-armed bandit models for the optimal design of clinical trials:
  benefits and challenges.
\newblock \emph{Statistical science: a review journal of the Institute of
  Mathematical Statistics}, 30\penalty0 (2):\penalty0 199, 2015.

\bibitem[Wei(1985)]{wei1985asymptotic}
Ching-Zong Wei.
\newblock Asymptotic properties of least-squares estimates in stochastic
  regression models.
\newblock \emph{The Annals of Statistics}, pages 1498--1508, 1985.

\bibitem[Zhang and Zhang(2014)]{zhang2014confidence}
Cun-Hui Zhang and Stephanie~S Zhang.
\newblock Confidence intervals for low dimensional parameters in high
  dimensional linear models.
\newblock \emph{Journal of the Royal Statistical Society: Series B (Statistical
  Methodology)}, 76\penalty0 (1):\penalty0 217--242, 2014.

\end{thebibliography}

\iftoggle{icml}{
\clearpage
  \appendix
  \onecolumn 
\section{Proofs of main results}
\iftoggle{icml}{
\label{app:proofs}
}{}
\ifboolexpr{togl{arxiv}}
{
\label{sec:proofs}
}{}

\subsection{Proofs of Theorems \ref{thm:variance} and \ref{thm:vardom}}
The proofs of the main results rely on the following
simple lemma.

\begin{lemma}\label{lem:Wclosedform}
Consider the $\vect{W}$-estimate as defined in Algorithm
\ref{alg:decorrmethod}. Assume $\twonorm{\vect{x}_i}^2 \le C$. Then for any $i$,
\begin{align}
\norm{\id_p - \vect{W}_{i-1} \vect{X}_{i-1}}_F^2 -
\norm{\id_p - \vect{W}_i \vect{X}_{i}}^2_F &\asymp 2\lambda(n){\twonorm{\vect{w}_i}^2}
\end{align}
\end{lemma}
\begin{proof}
This follows directly from the fact that
$\vect{W}_i\vect{X}_i = \vect{W}_{i-1}\vect{X}_{i-1} + \vect{w}_i\vect{x}_i^\sT$ and 
the following formula for $\vect{w}_i$:
\begin{align}
\vect{w}_i &= \frac{(\id_p - \vect{W}_{i-1}\vect{X}_{i-1})\vect{x}_i}{\lambda(n) + \twonorm{\vect{x}_i}^2}
\end{align}
which implies:
\begin{align}
\norm{\id_p - \vect{W}_{i-1} \vect{X}_{i-1}}_F^2 -
\norm{\id_p - \vect{W}_i \vect{X}_{i}}^2_F &=  (2\lambda(n) + \twonorm{\vect{x}_i}^2){\twonorm{\vect{w}_i}^2}
\end{align}
The result follows as $\twonorm{\vect{x}_i}^2$ is bounded uniformly. 
\end{proof}

We can now prove Theorems \ref{thm:variance} and \ref{thm:vardom} in a 
straightforward fashion.
\begin{proof}[Proof of Theorem \ref{thm:variance}]
We have:
\begin{align}
\Tr\{\Var(\sv)\} &= \sigma^2 \E\Big\{\sum_i \twonorm{\vect{w}_i}^2\Big\} \\
&\asymp \frac{\sigma^2}{2\lambda(n) }\Big( \norm{\id_p}_F^2 - \E\big\{\norm{\id_p - \vect{W}_n \vect{X}_n}_F^2\big\} \Big),
\end{align}
where in the second line we use Lemma \ref{lem:Wclosedform} and sum
over the telescoping series in $i$. The result follows. 
\end{proof}

\begin{proof}[Proof of Theorem \ref{thm:vardom}]

From Lemma \ref{lem:biasvardecomp}, the definition of the spectral norm $\opnorm{\cdot}$, and Cauchy-Schwarz
we have that
\begin{align}
\statictwonorm{\beta - \E\{\betahat\}}^2 
&\le \E\{ \opnorm{\id_p - \vect{W}_n\vect{X}_n}^2\}
\E\{\statictwonorm{\betahat_\ols - \beta}^2\}. 
\end{align}
Using Theorem \ref{thm:lai}, the second term is bounded by $p\sigma^2 \E\{\log\lambda_{\max}/\lambda_{\min}\}$. 
We first show that this term is at most  $p\sigma^2 \log n /\lambda(n)$, under the conditions
of Theorem \ref{thm:vardom}. First, note that
\begin{align}
 \lambda_{\max} &\le \Tr\{ \vect{X}_n\vect{X}_n^\sT\}\\
 &\le \sum_i \twonorm{\vect{x}_i}^2 
 \le C^2n.
 \end{align} 
With this and condition \cref{eq:vardomcond2}, we have that:
\begin{align}
\E\bigg\{\frac{\log(\lambda_{\max})}{\lambda_{\min}}\bigg\} &\le \E\Big(\frac{\log n}{\lambdamin(\bX_n^\sT\bX_n)} \Big) \\
&=  \frac{\log n} {\lambda(n)} + O\Big(\frac{1}{n}\Big) \\
&= O\Big(\frac{\log n}{\lambda(n)} \Big).   
\end{align}
Therefore, $\E\{\lVert \betahat_\ols - \beta\rVert^2\}= O(p\sigma^2\log n/\lambda(n))$.
By condition \cref{eq:vardomcond1} we have that the bias satisfies:
\begin{align}
\lVert \beta - \E\{\betahat\} \rVert^2
&= o\Big(\frac{p\sigma^2}{\lambda(n)}\Big).
\end{align}
On the other hand, for the variance, Theorem \ref{thm:variance} 
yields
\begin{align}
\Tr(\Var(\sv)) &= \frac{p\sigma^2}{\lambda(n)} \Big( 1 - \frac{\E\{\norm{\id_p - \vect{W}_n\bX_n}_F^2\}}{p}\Big) =\Theta\Big(\frac{p\sigma^2}{\lambda(n)}\Big).
\end{align}
provided $\E\{\norm{\id_p - \vect{W}_n\bX_n}_F^2\}/p \to 0$.
Condition \cref{eq:vardomcond1} guarantees that
\begin{align}
\frac{\E\{\norm{\id_p - \vect{W}_n\bX_n}_F^2\}}{p}
&\le \E\{\opnorm{\id_p- \vect{W}_n \bX_n}^2\} = o(1/\log n).
\end{align}
This finishes the proof.

\end{proof}

We split the proof of Proposition \ref{prop:sufficient} for
the different conditions independently in the following lemmas. 

\begin{lemma}\label{lem:biasboundexplore}
Suppose that the data collection process satisfies \cref{eq:explorecond1} and \cref{eq:explorecond2}.
Then for any $\lambda\ge 1$
we have that:
\begin{align}
\E\big\{ \norm{\id_p - \vect{W}_n\vect{X}_n}_F^2\big\}
&\le p\exp\Big( -\frac{n\mubar_n}{\lambda}\Big)
\end{align}
\end{lemma}
\begin{proof}
Define $\vect{M}_i = \id_p - \vect{W}_{i}\vect{X}_i$. Then,
from Lemma \ref{lem:Wclosedform} and the closed form 
for $\vect{w}_i$ we have that:
\begin{align}
\norm{\vect{M}_{i-1}}_F^2 
-\norm{\vect{M}_{i}}^2_F  &= 
 \frac {2\lambda + \twonorm{\vect{x}_i}^2}{(\lambda + \twonorm{\vect{x}_i}^2)^2 }\Tr\{ \vect{M}_{i-1}\vect{x}_i\vect{x}_i^\sT \vect{M}_{i-1}^\sT \} \\
&\ge \frac{1}{\lambda + \twonorm{\vect{x}_i}^2} 
\Tr\{ \vect{M}_{i-1}\vect{x}_i\vect{x}_i^\sT \vect{M}_{i-1}^\sT \}.
\end{align}
We now take expectations conditional on $\cG_{i-1}$ on both sides. 
Observing that
$(i)$ $\vect{W}_n$,  $\vect{X}_n$ and, therefore,  $\vect{M}_n$ are well-adapted and $(ii)$ using condition \cref{eq:explorecond1}, we have
\begin{align}
\E\{\norm{\vect{M}_{i-1}}_F^2 |\cG_{i-1}\}
-\E\{\norm{\vect{M}_{i}}^2_F | \cG_{i-1}\}%
&\ge  \frac{\mu_i(n)}{\lambda} \E\{ \norm{\vect{M}_{i}}_F^2 |\cG_{i-1}\}, \\
\text{ or } \,\, \E\{\norm{\vect{M}_{i}}^2_F | \cG_{i-1}\}
&\le  
\exp \Big(\frac{-\mu_i(n)}{\lambda}\Big)\E\{\norm{\vect{M}_{i-1}}_F^2 |\cG_{i-1}\}.
\end{align}
Removing the conditioning on $\cG_{i-1}$ and iterating over $i=1, 2, \dots, n$ gives the claim.
\end{proof}

\begin{lemma}\label{lem:biasboundcommute}
If the matrices $\{\vect{x}_i\vect{x}_i^\sT\}_{i\le n}$ commute, we have
that
\begin{align}
\opnorm{\id_p - \vect{W}_n\vect{X}_n} &\le \exp\Big(-\frac{\lambda_{\min}}{\lambda}\Big)
\end{align}
\end{lemma}
\begin{proof}
From the closed form in Lemma \ref{lem:Wclosedform} and induction, 
we get that:
\begin{align}
\id_p - \vect{W}_n\vect{X}_n &=
\prod_{i\le n} \Big(\id_p - \frac{\vect{x}_i\vect{x}_i^\sT}{\lambda + \twonorm{\vect{x}_i}^2}\Big).
\end{align}
The scalar
equality $\exp(a + b) = \exp(a)\exp(b)$ extends to 
commuting matrices $\vect{A}, \vect{B}$. Applying this
to the terms in the product above, which commute by assumption:
\begin{align}
\id_p - \vect{W}_n\vect{X}_n &= \exp\Big[\sum_i \log\Big (\id_p - \frac{\vect{x}_i\vect{x}_i^\sT}{\lambda + \twonorm{\vect{x}_i}^2}\Big)\Big]\\
&\mle \exp\Big( -\sum_i \frac{\vect{x}_i\vect{x}_i^\sT }{\lambda}\Big),
\end{align}
using the fact that $\exp(\log(1-a)) \le \exp(-a)$. Finally, {}employing
commutativity the fact that $\lambda_{\min}$ is the minimum 
eigenvalue of $\vect{X}_n^\sT\vect{X}_n = \sum_i \vect{x}_i\vect{x}_i^\sT$, the desired result follows. 
\end{proof}

We can now prove Proposition \ref{prop:sufficient}.

\begin{proof}[Proof of Proposition \ref{prop:sufficient}]
We need to satisfy conditions \cref{eq:vardomcond1} and \cref{eq:vardomcond2}
for both the cases. 
Using either Lemma \ref{lem:biasboundexplore} or \ref{lem:biasboundcommute}, with
the appropriate choice of $\lambda(n)$ we have that
\begin{align}
\E\{ \opnorm{\id_p - \vect{W}_n\vect{X}_n}^2\} &= o(1/\log n),
\end{align}
thus obtaining condition \cref{eq:vardomcond1}. 
In fact, this can be made polynomially small with a slightly smaller 
choice for $\lambda(n)$. Condition \cref{eq:vardomcond2} only needs to be verified
for the case of Lemma \ref{lem:biasboundexplore} or condition
\cref{eq:explorecond1}. It follows from
a standard application of the matrix Azuma inequality \cite{tropp2012user}, 
the fact that $n\mubar_n \ge K\sqrt{n}$ and the fact that $\twonorm{\vect{x}_i}^2$ 
are bounded. 
\end{proof}

\subsection{Proof of Theorem \ref{thm:martingaleclt}: Central limit theorem}

It suffices to show that, for every $t>0$:
\begin{align}
\lim_{n\to\infty} \E\{ e^{\im t S_n /\sqrt{V_n}}\} - e^{-t^2/2} &=0.   
\end{align}
Let $V_{i, n} = \sum_{j\le i} \E\{X_{j, n}^2 |\cF_{j-1, n}\}$
with $V_{0, n} = 0$. Therefore $V_{n, n} = V_n$ and $\E\{X_{i, n}^2 | \cF_{i-1, n} \} = V_{i, n} - V_{i-1, n}$. 
Let us also define the following error terms
\begin{align}
\nu_{i, n}^1 &= \E \Big\{ \frac{ \abs{\E\{ X_{i, n} | \cF_{i-1, n} \} - \E\{X_{i, n} | \cF_{i-1, n}, V_n\}} }{\sqrt{V_n}}   \Big\} \\
\nu^2_{i, n} &=
\E \Big\{ \frac{ \abs{\E\{ X^2_{n, n} | \cF_{n-1, n} \} - \E\{X^2_{n, n} | \cF_{n-1, n}, V_n\}} }{{V_n}}   \Big\} \\
\nu^3_{i, n} &= \E\Big\{ \frac{\abs{X_{i, n}}^3}{V_n^{3/2}}  \Big\} \\
\nu^4_{i, n}& = \E\Big\{\frac{\sigma_{i, n}^4}{V_n^2} \Big\}. 
\end{align}
The first two are moment stability, while the latter
two show that martingale increments are small.

Using the fact that $\abs{e^{\im x} - 1 -\im x +  x^2/2 } \le x^3$ and tower property, we have:
\begin{align}
\E\{ e^{\im t S_n /\sqrt{V_n}} \} &=  \E \{   \E \{ e^{\im t S_{n-1}/\sqrt{V_n}} e^{\im t X_{n, n}/\sqrt{V_n}} | \cF_{n-1, n} V_ n \}  \} \\
&=
\E \Big\{
e^{\im t S_{n-1} /\sqrt{V_n} } \Big( 1 + \im t \frac{X_{n, n}}{\sqrt{V_n}}  - \frac{t^2 X_{n, n}^2}{ 2V_n} \Big) \Big\} + O\Big( t^3 \E\Big\{ \frac{\abs{X_{n, n}}^3}{V_n^{3/2}} |  \Big\}\Big). 
\end{align}
Considering the first term, we write using tower property:
\begin{align}
\E\{ e^{\im t S_{n-1}/\sqrt{V_n}} \frac{X_{n, n}}{\sqrt{V_n}}\}
&= \E \{  \E\{  e^{\im t S_{n-1}/\sqrt{V_n}} \frac{\E\{ X_{n, n}  | \cF_{n-1, n}, V_n\}}{\sqrt{V_n}}  \} \\
&= \E\{ e^{\im t S_{n-1}/\sqrt{V_n} } \frac{\E\{ X_{n, n} | \cF_{n-1} \}} {\sqrt{V_n}}\} + O\Big(\E \Big\{ \frac{ \abs{\E\{ X_{n, n} | \cF_{n-1, n} \} - \E\{X_{n, n} | \cF_{n-1, n}, V_n\}} }{\sqrt{V_n}}   \Big\} \Big) \\
& = O\Big(\E \Big\{ \frac{ \abs{\E\{ X_{n, n} | \cF_{n-1, n} \} - \E\{X_{n, n} | \cF_{n-1, n}, V_n\}} }{\sqrt{V_n}}   \Big\} \Big) \\
&\equiv O( \nu^1_{n, n} ).
\end{align}
In an exactly analogous fashion:
\begin{align}
\E\{ e^{\im tS_{n-1}/\sqrt{V_n}} \Big(1 - \frac{t^2X^2_{n, n}}{2V_n}\Big)\}
&= \E \{  \E\{  e^{\im tS_{n-1}/\sqrt{V_n}} \Big( 1- \frac{t^2\E\{ X^2_{n, n}  | \cF_{n-1, n}, V_n\}}{{2V_n}} \Big) \} \\
&= \E\{ e^{\im tS_{n-1}/\sqrt{V_n} } \Big( 1- \frac{ t^2\E\{ X^2_{n, n} | \cF_{n-1} \}} {{2V_n}} \Big) \} + O\Big( t^2\E \Big\{ \frac{ \abs{\E\{ X^2_{n, n} | \cF_{n-1, n} \} - \E\{X^2_{n, n} | \cF_{n-1, n}, V_n\}} }{{V_n}}   \Big\} \Big) \\
&\equiv \E\{ e^{\im tS_{n-1} /\sqrt{V_n}} e^{-t^2 \sigma_{n, n}^2 /2 V_n}   \} + + O(t^4 \E\{ \sigma_{n, n}^4/V_n^2 \}) + O(t^2\nu^2_{n, n}).
\end{align}
Using these estimates, we obtain:
\begin{align}
\E\{e^{\im tS_n /\sqrt{V_n}}\} &= \E\{ e^{\im tS_{n-1}/\sqrt{V_n} } e^{-t^2\sigma_{n, n}^2/2V_n} \} + O \big(t \nu^1_{n, n} + t^2 \nu^2_{n, n} +t^3 \nu^3_{n, n} +  t^4 \nu^4_{n, n} \big). 
\end{align}
At this point, we iterate the argument, accumulating error terms. 
The only minor difference is that we have to be 
more careful about the conditioning. 

We start with the main term on the RHS.
\begin{align}
\E\{ e^{\im t S_{n-1}/\sqrt{V_n} - t^2\sigma_{n, n}^2/2V_n }  \}
&= \E \{   e^{\im t S_{n-2}/\sqrt{V_n} - t^2\sigma_{n, n}^2/2V_n }
e^{\im t X_{n-1, n}/\sqrt{V_n}} \} \\
& = \E  \{  e^{\im t S_{n-2}/\sqrt{V_n} - t^2\sigma_{n, n}^2/2V_n }
\E\{ e^{\im t X_{n-1, n}/\sqrt{V_n}} | \cF_{n-2, n}, V_n \}   \}.
\end{align}
In the final step, we use the tower property, along with the fact
that $(S_{n-2}, \sigma_{n, n}^2, V_n) = (S_{n-2}, V_{n-1,n} - V_n, V_n)$ are all 
measurable with respect to the minimal sigma algebra containing
$V_n, \cF_{n-2, n}$. Now, since the prefactor
$e^{\im t S_{n-2}/\sqrt{V_n} - t^2\sigma_{n, n}^4/2V_n } $ is
bounded in magnitude by 1, we can follow the same steps as before. 
\begin{align}
\E  \{  e^{\im t S_{n-2}/\sqrt{V_n} - t^2\sigma_{n, n}^2/2V_n }
\E\{ e^{\im t X_{n-1, n}/\sqrt{V_n}} | \cF_{n-2, n}, V_n \}   \}
& = \E \Big\{  e^{\im t S_{n-2}/\sqrt{V_n} - t^2\sigma_{n, n}^2/2V_n }
\Big(1 +  \frac{\im t X_{n-1, n}}{\sqrt{V_n}} -\frac{t^2X_{n-1, n}^2}{2V_n} \Big) \Big\} \\&\quad+ \E\{ \frac{t^3\abs{X_{n-1, n}}^3}{V_n^{3/2}}   \} \\
&= \E\{ e^{\im t S_{n-3}/\sqrt{V_n} - t^2(\sigma_{n-1, n}^2 + \sigma_{n, n}^2)/2V_n } \} \\&\quad+ O(t \nu^1_{n-1, n} + t^2 \nu^2_{n-1, n} t^3 \E\{\abs{X_{n-1, n}}^3/V_n^{3/2} + t^4 \E\{ \sigma_{n-1, n}^4/V_n^2\} \}).
\end{align}
At this point, we iterate the argument to obtain:
\begin{align}
\E\{ e^{\im t S_n /\sqrt{V_n}}\} & = e^{-t^2/2}
+ O\Big(\sum_{i}  t\nu^1_{i, n}+  t^2\nu^2_{i, n} + t^3 \nu^3_{i, n}
+ t^4\nu^4_{i, n}\Big). 
\end{align}
Our assumptions guarantee that each of the error terms 
vanish as $n\to\infty$, yielding the desired claim.

\subsection{Commutative problems: Proof of Propostion \ref{prop:commvarianceisstable}}

Here we assume that $\vect{x}_i \in\{\bv_1, \dots\bv_p\}$, 
a set of orthogonal  vectors. First, we compute some
closed form formulae that are useful in proving Proposition \ref{prop:commvarianceisstable}.

\begin{lemma}\label{lem:commutativevarformula}
Define the sequence $A =A(i)$ as the choice of arms at time $i$, $\vect{P}_a$ denote the orthogonal projector
along the direction $\bv_a$ and $r_a=  1- \norm{\bv_a}^2/(\lambda + \norm{\bv_a}^2)$.  
We have the following:
\begin{align}
\id_p - \vect{W}_i \vect{X}_i &= \sum_{a = 1}^p  \Big (\id - \frac{\bv_a \bv_a^\sT}{\lambda(n) + \norm{\bv_a}^2} \Big)^{N_a(i)} \vect{P}_a, \\
\vect{w}_i &=  \frac{ r_a^{N_A(i-1)} }{\lambda(n) + \norm{\bv_A}^2}  \bv_A. 
\end{align}
In particular, the variance is given by:
\begin{align}
 \vect{W}_n \vect{W}_n ^\sT= \sum_{i=1}^n \vect{w}_i \vect{w}_i^\sT & =
   \sum_{a =1}^p \frac{1 - r_a^{2N_a(n-1)+1}}{(\lambda(n) + \norm{\bv_a}^2)^2 ( 1- r_a^2)}  \bv_a \bv_a^\sT. 
\end{align}
\end{lemma}

\begin{proof}[Proof of Proposition \ref{prop:commvarianceisstable}]
Below, we keep implicit the dependence of $\lambda$ on $n$, 
with the understanding that $\lambda(n)$ diverges with $n$.
By Lemma \ref{lem:commutativevarformula}, we have 
\begin{align}
 \lambda \vect{W}_n \vect{W}_n^\sT 
&= \sum_{a=1}^p  c_a   \bv_a \bv_a^\sT, \\
\text{ where }  c_a &=  \frac{\lambda (1 - r_a^{2N_a(n-1) + 1})}{ (\lambda + \norm{\bv_a}^2)^2 (1-r_a^2)}
\end{align}
Note that, $r_a = 1 - \norm{\bv_a}/\lambda + O(1/\lambda^2)$
as $\lambda=\lambda(n)$ diverges,
which implies that 
\begin{align}
\lim_{n\to \infty} \frac{\lambda(n)}{(\lambda(n) + \norm{\bv_a}^2)^2 (1- r_a^2)} &= \frac{1}{2\norm{\bv_a}^2}.  
\end{align}
Therefore $c_a \to (2\norm{\bv_a}^2)^{-1}$ in $L_1$ provided
$r_a^{2N_a(n-1)} \to 0$ in $L_1$.
To show this:
\begin{align}
r_a^{2N_a(n-1)} &= \Big( 1- \frac{\norm{\bv_a}^2}{\lambda+ \norm{\bv_a}^2}\Big)^{2N_a(n-1)}\\
&\le \exp\Big( - 2\frac{\norm{\bv_a}^2N_a(n-1)}{\lambda + \norm{\bv_a}^2} \Big)  \\
 &\le \exp\Big( - 2\frac{\norm{\bv_a}^2N_a(n) - \norm{\bv_a}^2}{\lambda + \norm{\bv_a}^2} \Big) \\
& \le \exp\Big(-2\frac{\lambdamin(\bX_n^\sT\bX_n) - C^2}{\lambda+C^2} \Big),
\end{align}
where in the last line we use the fact that $\norm{\bv_a}\le C\equiv \max_a\norm{\bv_a}$.
and:
\begin{align}
\lambda_{\min}(\bX_n^\sT\bX_n) &= \lambdamin\Big( \sum_{a=1}^p \vect{P}_a \norm{\bv_a}^2 N_a(n)   \Big) = \min_a \norm{\bv_a}^2 N_a(n).
\end{align}
Since $\lambda(n) /\lambdamin(\bX_n^\sT\bX_n) \to 0$ in probability, 
$r_a^{2N_a(n-1)}\to 0$ in probability and therefore, also
in $L_1$ using bounded convergence. 

It follows that 
\begin{align}
\lambda(n)\vect{W}_n\vect{W}_n^\sT &\stackrel{L_1}{\to}
\sum_{a=1}^p \frac{\bv_a\bv_a^\sT}{2\norm{\bv_a}^2} = 
\sum_{a=1}^p \frac{\vect{P}_a}{2} = \frac{\id_p}{2}. 
\end{align}

\end{proof}

It remains to prove Lemma \ref{lem:commutativevarformula}. 
\begin{proof}[Proof of Lemma \ref{lem:commutativevarformula}]
From Lemma \ref{lem:Wclosedform} and induction we have that
\begin{align}
\id_p - \vect{W}_i\bX_i &= \prod_{j\le i} \Big(\id - \frac{\bx_j \bx_j^\sT}{\lambda + \norm{\bx_j}^2}\Big).
\end{align}
Since the matrices $\bx_j \bx_j^\sT$ commute, we can rearrange the 
product as
\begin{align}
\id_p - \vect{W}_i\bX_i &= \prod_{a = 1}^p \prod_{j \le i } 
\Big(\id - \frac{\bx_j \bx_j^\sT}{\lambda + \norm{\bx_j}^2}\Big)^{\ind(A(j) = a)} \\
& = \prod_{a = 1}^p  \Big (\id - \frac{\bv_a \bv_a^\sT}{\lambda + \norm{\bv_a}^2} \Big)^{N_a(i)}. 
\end{align}
If $\vect{P}_a$ is the orthogonal projector along $\vect{v}_a$ (i.e., 
$\vect{P}_a = \bv_a\bv_a^\sT /\norm{\bv_a}^2$), then
\begin{align}
(\id_p - \vect{W}_i \vect{X}_i)\vect{P}_a
&= \Big(\id_p  - \frac{\bv_a \bv_a^\sT}{\lambda + \norm{\bv_a}^2} \Big)^{N_a(i)} \vect{P}_a  \\
& = r_a^{N_a(i)} \vect{P}_a .  
\end{align} 
Since $\sum_a \vect{P}_a = \id_p$, the first claim follows. 
The formula for $\vect{w}_i$ follows immediately from
this decomposition of $\id_p - \vect{W}_i \vect{X}_i$.  

For the variance, we have
\begin{align}
\sum_{i=1}^n \vect{w}_i \vect{w}_i & = \sum_{i=1}^n \frac{r_A^{2N_A(i-1)}}{(\lambda + \norm{\bv_A}^2)^2} \bv_A \bv_A^\sT\\
& = \sum_{a = 1}^p \frac{1 + r_a^2 + \dots r_a^{2N_a(n-1)}}{(\lambda + \norm{\bv_a}^2)^2}     \bv_a\bv_a^\sT \\
&= \sum_{a =1}^p \frac{1 - r_a^{2N_a(n-1)+1}}{(\lambda + \norm{\bv_a}^2)^2 ( 1- r_a^2)}  \bv_a \bv_a^\sT. 
\end{align}
\end{proof}

 \newpage
\section{Supplementary experiments}
\label{sec:supplexpts}
In Figure \ref{fig:ar1} we show the results of applying $\vect{W}$-decorrelation
to the AR(1) process:
\begin{align}
y_{i} & = \beta_{0} y_{i-1} + \eps_i. 
\end{align}
As in the main text
the we run 4000 Monte Carlo iterations of generating
a length $n=100$ time series with
$\eps_i\sim\Unif([-1, 1])$ i.i.d. and $\beta = 1$, generate
the $\ols$ estimate, our decorrelated version and plot
empirical upper tail coverage, lower tail coverage as well
as probability plots. As evidenced by Figure \ref{fig:ar2}, $\vect{W}$-decorrelation
provides estimates with  valid
empirical coverage. Simulataneously, it provides smaller widths than 
concentration inequalities, particularly in the moderate confidence regime. 
\begin{figure}[ht]
\centering
\includegraphics[width=0.32\linewidth]{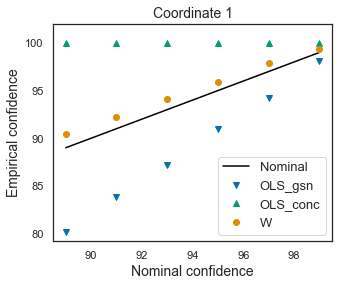}
\includegraphics[width=0.32\linewidth]{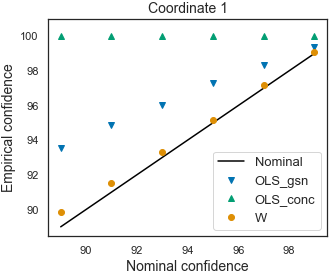}
\includegraphics[width=0.32\linewidth]{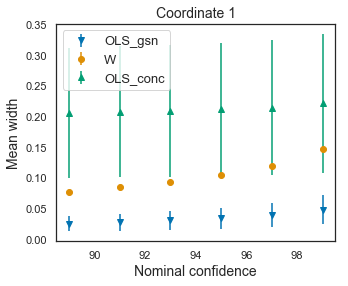}
\includegraphics[width=0.49\linewidth]{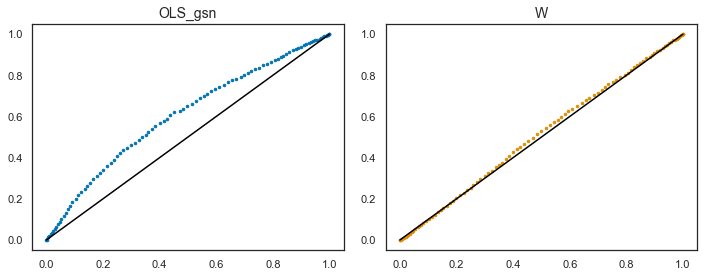} 
\includegraphics[width=0.49\linewidth]{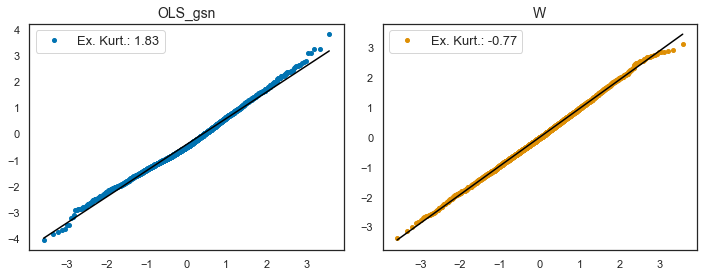}
	\caption{
Top: Lower (left) and upper (middle) coverage probabilities for 
$\ols$ with Gaussian intervals, $\ols$ with concentration inequality
 intervals, and decorrelated $\vect{W}$-decorrelated estimate
  intervals.  Note that `$\ols_{\rm conc}$' has always 100\% coverage. 
  Mean confidence widths (right) for various estimators. The 
  error bars show one (empirical) standard deviation. 
	Bottom: PP (left) and QQ (right) for the distribution of 
	errors of standard $\ols$ estimate and the $\vect{W}$-decorrelated estimate. } \label{fig:ar1}
\end{figure}

  }{}

\ifboolexpr{togl{arxiv}}{
\appendix
\numberwithin{equation}{section}

}{}

\end{document}